%% file: main.tex
\documentclass[12pt]{article}

\usepackage[top=0.95in,bottom=0.95in,left=0.95in,right=0.95in]{geometry}
\usepackage[utf8]{inputenc} 

\usepackage[T1]{fontenc}   
\usepackage{lmodern}
\usepackage{hyperref}      
\usepackage{url}            
\usepackage{booktabs}       
\usepackage{amsfonts}       
\usepackage{nicefrac}       
\usepackage{microtype}
\usepackage{graphicx}
\usepackage{subfigure}
\usepackage{booktabs} 
\usepackage{enumitem}
\usepackage{makecell}
\usepackage{textcomp}
\usepackage{xcolor}
\usepackage[round]{natbib}

\input{preamble}

\begin{document}

\title{Optimizing Black-box Metrics with\\Iterative Example Weighting}

\author{
Gaurush Hiranandani\\
UIUC\\
\texttt{gaurush2@illinois.edu}\\
\and
Jatin Mathur\\
UIUC\\
\texttt{jatinm2@illinois.edu}\\
\and
Harikrishna Narasimhan \\
Google Research USA\\
\texttt{hnarasimhan@google.com}
\and
Mahdi Milani Fard \\
Google Research USA\\
\texttt{mmilanifard@google.com}\\
\and
Oluwasanmi Koyejo\\
Google Research Accra \& UIUC\\
\texttt{sanmik@google.com}\\
}

\date{\today}

\flushbottom
\maketitle

\input{abstract}
\input{introduction}
\input{setup}
\input{methods}
\input{theory}
\input{relatedwork}
\input{experiments}

\input{discussion}
\input{ack}

\bibliography{main}
\bibliographystyle{plainnat}

\clearpage

\input{supplement}

\end{document}

%% file: preamble.tex
\usepackage{times}
\usepackage{helvet}
\usepackage{courier}
\usepackage{amssymb}
\usepackage{amsmath}
\usepackage{amsthm}
\usepackage{latexsym}
\usepackage{graphicx}
\usepackage{wrapfig}
\usepackage{color,colortbl}
\usepackage{algorithm,algorithmic}
\usepackage{url}
\usepackage{subfigure}
\usepackage[font=small]{caption}
\usepackage{tikz}
\usepackage{wrapfig}
\usepackage[toc, page, header]{appendix}
\hypersetup{colorlinks,urlcolor=black, citecolor=blue}

\newtheorem{thm}{Theorem}
\newtheorem{lem}[thm]{Lemma}
\newtheorem{asp}{Assumption}

\newtheorem{exmp}{Example}

\newtheorem*{thm*}{Theorem}
\newtheorem*{prop*}{Proposition}
\newtheorem*{lem*}{Lemma}

\definecolor{Gray}{gray}{0.8}
\renewcommand{\hat}{\widehat}
\renewcommand{\tilde}{\widetilde}
\renewcommand{\>}{{\rightarrow}}

\newcommand{\argmax}{\operatorname{argmax}}
\newcommand{\argmin}{\operatorname{argmin}}

\newcommand{\R}{{\mathbb R}}

\newcommand{\N}{{\mathbb N}}
\renewcommand{\P}{{\mathbf P}}
\newcommand{\E}{{\mathbf E}}

\newcommand{\1}{{\mathbf 1}}

\newcommand{\A}{{\mathbf A}}

\newcommand{\cC}{{\mathcal C}}
\newcommand{\C}{{\mathbf C}}
\newcommand{\cE}{{\mathcal E}}

\renewcommand{\H}{{\mathcal H}}

\renewcommand{\L}{{\mathbf L}}

\newcommand{\X}{{\mathcal X}}
\newcommand{\Y}{{\mathcal Y}}

\renewcommand{\c}{{\mathbf c}}

\newcommand{\x}{{\mathbf x}}

\newcommand{\lin}{\textup{\textrm{lin}}}

\newcommand{\balpha}{{\boldsymbol \alpha}}

\newcommand{\bSigma}{{\boldsymbol \Sigma}}

\newcommand{\T}{{\mathbf T}}

\newcommand{\perf}{\cE}

\newcommand{\Dtrue}{D}
\newcommand{\Dshift}{\mu}
\newcommand{\tr}{\textup{\textrm{tr}}}
\newcommand{\val}{\textup{\textrm{val}}}
\newcommand{\onehot}{\textup{\textrm{onehot}}}

\newcommand{\W}{\mathbf{W}}
\newcommand{\D}{\mathbf{D}}
\newcommand{\bPhi}{\boldsymbol{\Phi}}

\newcommand{\bD}{\mathbf{D}}
\newcommand{\bcE}{\boldsymbol{\cE}}
\newcommand{\diag}{\textup{\textrm{diag}}}

\newcommand{\bbeta}{\boldsymbol{\beta}}

\newcommand{\bupsilon}{\boldsymbol{\upsilon}}
\newcommand{\bE}{\mathbf{E}}

\newcommand\numberthis{\addtocounter{equation}{1}\tag{\theequation}}

\renewcommand\cite{\citep}

%% file: abstract.tex
\begin{abstract}
We consider learning to optimize a classification metric defined by a black-box function of the confusion matrix. Such black-box learning settings are ubiquitous, for example, when the learner only has query access to the metric of interest, or in noisy-label and domain adaptation applications where the learner must evaluate the metric via performance evaluation using a small validation sample. Our approach is to adaptively learn example weights on the training dataset such that the resulting weighted objective best approximates the metric on the validation sample. We show how to model and estimate the example weights and use them to iteratively post-shift a pre-trained class probability estimator to construct a classifier. We also analyze the resulting procedure's statistical properties. Experiments on various label noise, domain shift, and fair classification setups confirm that our proposal compares favorably to the state-of-the-art baselines for each application.
\end{abstract}

%% file: introduction.tex
\section{Introduction}
\label{sec:introduction}

In many real-world machine learning tasks, the evaluation metric one seeks to optimize is not explicitly available in closed-form. This is true for metrics that are evaluated through live experiments or by querying human users \cite{tamburrelli2014towards, hiranandani2018eliciting}, or that require access to private or legally protected data \cite{awasthi+21},
and hence cannot be written as an explicit training objective. This is also the case when the learner only has access to data with skewed training distribution or labels with heteroscedastic noise~\cite{huang2019addressing,jiang2020optimizing}, and hence cannot directly optimize the metric on the training set despite knowing its mathematical form. 

These problems can be framed as black-box learning tasks, where the goal is to optimize an unknown classification metric on a large (possibly noisy) training data, given  access to evaluations of the metric on a small, clean validation sample \citep{jiang2020optimizing}. Our high-level approach to these learning tasks is to adaptively assign weights to the training examples, so that the resulting weighted training objective closely approximates the black-box metric on the validation sample. We then construct a classifier by using the example weights to  post-shift a class-probability estimator pre-trained on the training set. This results in an efficient, iterative approach that does not require any re-training. 

Indeed, example weighting strategies have been  widely used  to both optimize metrics and to correct for distribution shift, but prior works either handle specialized forms of metric or data noise~\cite{sugiyama2008direct, natarajan2013learning, patrini2017making}, formulate the example-weight learning task as a difficult non-convex problem that is hard to analyze \cite{ren2018learning,zhao2019metric}, or employ an expensive surrogate re-weighting strategy that comes with limited statistical guarantees~\cite{jiang2020optimizing}. 
In contrast, we propose a simple and effective approach to optimize a general black-box metric (that is  a function of the confusion matrix)  and provide a rigorous statistical analysis. 

A key element of our approach is eliciting the weight coefficients by probing the  black-box metric at few select classifiers and solving a system of linear equations matching the weighted training errors to the validation metric. 
We choose the ``probing'' classifiers so that the linear system is well-conditioned, for which we provide both theoretically-grounded options and practically efficient variants. This weight elicitation procedure is then used as a subroutine to iteratively construct the final plug-in classifier.

\textbf{Contributions:} (i) We provide a method for eliciting example weights for linear black-box metrics (Section \ref{sec:example-weights}). (ii) We  use  this procedure to iteratively learn a plug-in classifier for general black-box metrics (Section \ref{sec:algorithms}). (iii) We provide theoretical guarantees for metrics that are concave functions of the confusion matrix under distributional assumptions
(Section \ref{sec:theory}). (iv) We experimentally show that our approach is competitive with (or better than) the state-of-the-art methods for tackling label noise in CIFAR-10~\cite{krizhevsky2009learning} and domain shift in Adience~\cite{eidinger2014age}, and optimizing with proxy labels and a black-box fairness metric on Adult~\cite{Dua:2019} (Section \ref{sec:experiments}). 

\textbf{Notations:} $\Delta_m$ denotes the $(m-1)$-dimensional simplex. $[m] = \{1,\dots,m\}$ represents an index set. $\onehot(j) \in \{0,1\}^m$  returns the one-hot encoding of  $j \in [m]$. The $\ell_2$ norm a vector is denoted by $\|\cdot\|$. 

%% file: setup.tex
\vspace{-0.4cm}
\section{Problem Setup}
\label{sec:setup}

We consider a standard multiclass setup with an instance space  $\X \subseteq \R^d$ and a label space $\Y = [m]$. We wish to learn a randomized multiclass classifier $h: \X \> \Delta_m$ that for any input $x \in \X$ predicts a distribution $h(x) \in \Delta_m$ over the $m$ classes. We will also consider deterministic classifiers $h: \X \>[m]$ which map an instance $x$ to one of $m$ classes.

\textbf{Evaluation Metrics.} Let $\Dtrue$ denote the underlying data distribution over $\X \times \Y$. 
We will evaluate the performance of a classifier $h$ on $\Dtrue$ using an evaluation metric $\perf^\Dtrue[h]$, with higher values indicating better performance. Our goal is to  learn a classifier $h$ that maximizes this evaluation measure:
\begin{equation}
\textstyle
\max_{h}\,\perf^{\Dtrue}[h].
\label{eq:unconsrained}
\end{equation}

We will  focus on metrics $\perf^\Dtrue$ that can be written in terms of  classifier's  confusion matrix $\C[h] \in [0,1]^{m\times m}$, where the $i,j$-th entry is the probability that the true label is $i$ and the randomized classifier $h$ predicts $j$:
\[
C^{\Dtrue}_{ij}[h] = \E_{(x, y) \sim \Dtrue}\left[\1(y = i)h_j(x)\right].
\]

The performance of the classifier can then be evaluated using a (possibly unknown) function $\psi: [0,1]^{m\times m}\>\R_+$ of the confusion matrix: 
\begin{equation}
\perf^D[h] = \psi(\C^D[h]).
\label{eq:perf-conf}
\end{equation}
Several common classification metrics take this form, including typical linear metrics $\psi(\C) \,=\, \sum_{ij}L_{ij}\,C_{ij}$ for some reward matrix $\L \in \R_+^{m\times m}$,  the  F-measure {\small $\psi(\C) \,=\,\sum_i \frac{2C_{ii}}{\sum_j C_{ij} + \sum_j C_{ji}}$} \cite{Lewis95}, and the G-mean {\small$\psi(\C) = \big(\prod_i \big({C_{ii}}/\sum_j C_{ij}\big)\big)^{1/m}$} \cite{Daskalaki+06}.

We consider settings where the learner has query-access to the evaluation metric $\perf^D$, i.e., can evaluate the metric for any given classifier $h$ but cannot directly write out the metric as an explicit mathematical objective. This happens when the metric is truly a black-box function, i.e., $\psi$ is unknown, or when $\psi$ is known, but we have access to only a noisy version of the distribution $D$ needed to compute the metric. 

\textbf{Noisy Training Distribution.} 
For learning a classifier, we assume access to a large  sample $S^{\tr}$ of $n^{\tr}$ examples drawn from a distribution $\Dshift$, which we will refer to as the ``training'' distribution. The training distribution $\Dshift$ may be the same as the true  distribution $\Dtrue$, or 
may differ from the true  distribution $\Dtrue$ 
in  the  feature distribution $\P(x)$, the conditional label distribution $\P(y|x)$, or both. 
 We also assume access to a  smaller sample $S^{\val}$ of
$n^{\val}$ examples
drawn from the true distribution
$\Dtrue$. 
We will refer to the sample $S^{\tr}$ as the ``training'' sample, and the smaller sample $S^{\val}$ as the ``validation'' sample. We seek to solve \eqref{eq:unconsrained} using both these  samples. 

The following are some examples of noisy training distributions in the literature:
\begin{exmp}[Independent label noise (ILN) \cite{natarajan2013learning, patrini2017making}]
\label{ex:iln}
The distribution  $\Dshift$ draws an example $(x,y)$ from $\Dtrue$, and randomly flips $y$ to $\tilde{y}$ with probability $\P(\tilde{y}|y)$, independent of the instance $x$.
\end{exmp}
\begin{exmp}[Cluster-dependent label noise (CDLN) \cite{wang2020fair}]
Suppose each  $x$ belongs to one of $k$ disjoint clusters $g(x) \in [k]$. The distribution  $\Dshift$ draws  $(x,y)$ from $\Dtrue$ and randomly flips $y$ to $\tilde{y}$ with probability $\P(\tilde{y}|y,g(x))$.
\end{exmp}
\begin{exmp}[Instance-dependent label noise (IDLN) \cite{menon2018learning}] 
$\Dshift$ draws  $(x,y)$ from $\Dtrue$ and randomly flips $y$ to $\tilde{y}$ with probability $\P(\tilde{y}|y,x)$, which may depend on  $x$.
\label{ex:instnoise}
\end{exmp}
\begin{exmp}[Domain shift (DS) \cite{sugiyama2008direct}] 
\label{ex:ds}
$\Dshift$ draws $\tilde{x}$ according to a distribution $\P^{\Dshift}({x})$ different from $\P^{\Dtrue}({x})$, but draws $y$ from the true conditional $\P^{\Dtrue}(y|\tilde{x})$.
\end{exmp}

\begin{table}[t]
    \centering
    \caption{Example weights $\W: \X \> \R_+^{m\times m}$ for linear metric {\small $\perf^\Dtrue[h]=\langle\L, \C^\Dtrue[h]\rangle$} under the  noise models in Exmp.\ \ref{ex:iln}--\ref{ex:ds}, where $W_{ij}(x)$ is the weight on entry $C_{ij}$. In Sec.\ \ref{sec:example-weights}--\ref{sec:algorithms}, we consider metrics that are functions of the diagonal confusion  entries alone (i.e.\ $\L$ and $\T$ are diagonal), and handle  general  metrics in  Appendix \ref{app:linear-gen}.}
    \vskip -0.1cm
    \small{
    \begin{tabular}{ccc}
    \hline
    Model &
         Noise Transition Matrix & Correction Weights\\
        \hline
        ILN 
        & $T_{ij} = \P(\tilde{y}=j|y=i)$& 
        $\W(x) = \L \odot \T^{-1}$
        \\
        CDLN 
        & $T^{[k]}_{ij} = \P(\tilde{y}=j|y=i, g(x)=k)$& 
        $\W(x) = 
        \L \odot (\T^{[g(x)]})^{-1}$
        \\
        IDLN 
        & $T_{ij}(x) = \P(\tilde{y}=j|y=i, x)$& 
        $\W(x) = \L \odot (\T(x))^{-1}$
        \\
        DS & 
        - &  $W_{ij}(x) = {\P^\Dtrue(x)}/{\P^\Dshift(x)}, \forall i,j$\\
        \hline
    \end{tabular}
    }
      \label{tab:correction-weights}
\end{table}

Our approach is to learn example weights on the training sample $S^\tr$, so that the resulting weighted empirical objective (locally, if not globally) approximates an estimate of the metric $\perf^\Dtrue$ on the validation sample $S^\val$. For ease of presentation, we will assume that the metrics only depend on the diagonal entries of the confusion matrix, i.e., $C_{ii}$'s. In Appendix \ref{app:linear-gen}, we elaborate how our ideas can be extended to handle metrics that depend on the entire confusion matrix. 

While our approach uses randomized classifiers, in practice one can replace them with similarly performing deterministic classifiers using, e.g., the techniques of \cite{cotter19stochastic}. In what follows, we will need the empirical confusion matrix on the validation set $\hat{\C}^\val[h]$, where  $\hat{C}_{ij}^{\val}[h] = \frac{1}{n^\val}\sum_{(x, y) \in S^\val}\1(y=i)h_j(x)$. 

%% file: methods.tex
\section{Example Weighting for  Linear Metrics}
\label{sec:example-weights}

We first describe our example weighting strategy for linear functions of the diagonal entries of the confusion matrix, which is given by:
\begin{equation}
  \textstyle \perf^\Dtrue[h] \,=\, 
\sum_{i}\beta_i\,C^D_{ii}[h]
\label{eq:linearmetric}
\end{equation}
\vskip -0.2cm
for some (unknown) weights $\beta_1,\ldots,\beta_m$. In the next section, we will discuss how to use this procedure as a subroutine to handle more complex metrics.

\subsection{Modeling Example Weights} 

We define an example weighting function $\W: \X  \> \R^{m}_+$ which associates $m$ \emph{correction weights} $[W_{i}(x)]_{i=1}^m$ with each example $x$ so that: 
\begin{equation}
\textstyle
\E_{(x, y) \sim \Dshift}\Big[\sum_{i} W_{i}(x)\,\1(y = i)h_i(x)\Big] \,\approx\, 
\perf^\Dtrue[h],  \forall h.\hspace{-2pt}
\label{eq:example-weights}
\end{equation}
Indeed for the noise models in Examples \ref{ex:iln}--\ref{ex:ds}, there exist weighting functions $\W$ for which the above holds with equality.  Table \ref{tab:correction-weights} shows the form of the weighting function for general linear metrics.

Ideally, the weighting function $\W$ assigns $m$ independent weights for each example $x \in \X$. However, in practice, we estimate $\perf^\Dtrue$ using a small validation sample $S^\val \sim \Dtrue$. So to avoid having the example weights over-fit to the validation sample, we restrict the flexibility of $\W$ and set it to a  weighted sum of $L$ basis functions $\phi^\ell: \X \> [0,1]$:
\begin{equation}
\textstyle
W_{i}(x) \,=\, \sum_{\ell=1}^L \alpha^{\ell}_{i}\phi^\ell(x),
\label{eq:weighting}
\end{equation}
where $\alpha^{\ell}_{i} \in \R$ is the coefficient associated with basis function $\phi^\ell$ and diagonal confusion  entry $(i,i)$. 

In practice, the basis functions can be as simple as a  partitioning of the instance space into $L$ clusters, i.e.,:
\begin{equation}
\phi^\ell(x) =  \1(g(x) = \ell),
\label{eq:hardlcuster}
\end{equation}
for a clustering function $g: \X\>[L]$,
or may define a more complicated soft clustering using, e.g., radial basis functions \cite{sugiyama2008direct} with centers
$x^\ell$ and width $\sigma$:
\begin{equation}
\phi^\ell(x) =  \text{exp}\left( -\Vert x - x^\ell \Vert/2\sigma^2 \right).
\label{eq:softlcuster}
\end{equation}

\subsection{$\phi$-transformed Confusions}

Expanding the weighting function in \eqref{eq:example-weights} gives us:
\begin{equation*}
\sum_{\ell=1}^L\sum_{i=1}^m
\alpha^{\ell}_{i}\,
\underbrace{
\E_{(x, y) \sim \Dshift}\big[ \phi^\ell(x)\,\1(y=i)h_i(x) \big]}_{\Phi_i^{\Dshift, \ell}[h]} \,\approx\, \perf^\Dtrue[h],  \forall h,
\end{equation*}
where $\bPhi^{\Dshift, \ell}[h] \in [0,1]^{m}$ can be seen as a $\phi$-transformed  confusion matrix  for the training distribution $\Dshift$. 
For example, if one had only one basis function 
$\phi^1(x) = 1,\forall x$, then $\Phi_{i}^{\Dshift, 1}[h] = \E_{(x, y) \sim \Dshift}\big[ \1(y=i)h_i(x)\big]$ gives the standard confusion entries for the training distribution. If the basis functions divides the data into $L$  clusters, as in \eqref{eq:hardlcuster}, then  $\Phi_{i}^{\Dshift, \ell}[h] = \E_{(x, y) \sim \Dshift}\big[\1(g(x)=\ell, y=i)h_i(x)\big]$  gives the training confusion entries evaluated on examples from cluster $\ell$. We can thus re-write equation~\eqref{eq:example-weights} as a weighted combination of the $\Phi$-confusion entries:
\begin{equation}
\sum_{\ell=1}^L\sum_{i=1}^m
\alpha^{\ell}_i
\Phi^{\Dshift, \ell}_i[h] \,\approx\, 
\perf^D[h], \forall h.
\label{eq:example-weights-reduced}
\end{equation}

\subsection{Eliciting Weight Coefficients $\balpha$}
\label{sec:estimating-weights}

We  next discuss how to estimate the weighting function coefficients $\alpha^\ell_{i}$'s from the training sample $S^\tr$ and validation sample $S^\val$. Notice that \eqref{eq:example-weights-reduced} gives a relationship between statistics $\bPhi^{\mu,\ell}$'s computed on the training distribution $\Dshift$, and the evaluation metric of interest computed on the true distribution $\Dtrue$. Moreover, for a fixed classifier $h$, the left-hand side is \textit{linear} in the unknown coefficients
$\balpha = [\alpha_1^1, \ldots, \alpha_1^L, \ldots, \alpha_m^1, \ldots, \alpha_m^L]\in \R^{Lm}$. 

We therefore probe the metric $\hat{\perf}^\val$ at $Lm$ different classifiers $h^{1,1}, \ldots, h^{1,m},\ldots,h^{L,1},\ldots, h^{L,m}$, which results in a set of $Lm$ linear equations of the form in \eqref{eq:example-weights-reduced}: 
\begin{align}
\textstyle\sum_{\ell,i}\alpha^{\ell}_i\,
\hat{\Phi}^{\tr, \ell}_i[h^{1,1}]&=
\hat{\perf}^\val[h^{1,1}],
\nonumber
\\
&\vdots\label{eq:system-of-equations-emp}\\
\textstyle\sum_{\ell,i}\alpha^{\ell}_i\,
\hat{\Phi}^{\tr, \ell}_i[h^{L,m}]&=
\hat{\perf}^\val[h^{L,m}],\nonumber
\end{align}
where $\hat{\Phi}_{i}^{\tr, \ell}[h] = \frac{1}{n^\tr}\sum_{(x, y) \in S^\tr}\phi^\ell(x)\,\1(y=i)h_i(x)$ is evaluated on the training sample and the metric $\hat{\perf}^\val[h] = \sum_{i}\beta_i\,\hat{C}^\val_{ii}[h]$ 
is evaluated on the validation sample.

\begin{figure}[t]
    \centering
    \includegraphics[scale=0.65]{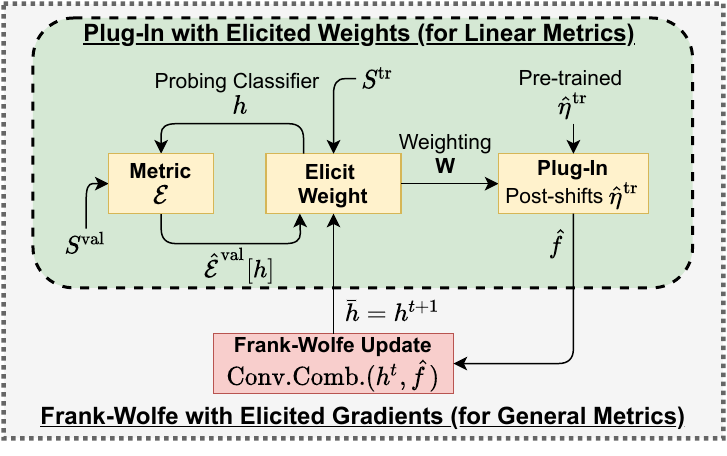}
    \caption{Overview of  our approach.}
    \label{fig:overview}
\end{figure}

More formally, let $\hat{\bSigma} \in \R^{Lm\times Lm}$ and $\hat{\bcE} \in \R^{Lm}$ denote the left-hand and right-hand side observations in \eqref{eq:system-of-equations-emp}, i.e.,:
\vspace{-0.2cm}
\begin{align*}
\textstyle
\hat{\Sigma}_{(\ell,i), (\ell',i')} &\,=\,
\frac{1}{n^\tr}\sum_{(x,y)\in S^\tr}\phi^{\ell'}(x)\1(y=i')h^{\ell,i}_{i'}(x), \\ \nonumber
\hat{\perf}_{(\ell,i)} &\,=\, \hat{\perf}^\val[h^{\ell,i}]. \numberthis 
\label{eq:perf-emp}
\end{align*}
Then the weight coefficients are given by  $\hat{\balpha} = \hat{\bSigma}^{-1}\hat{\bcE}$.

\subsection{Choosing the Probing Classifiers $h^{1,1},\ldots,h^{L,m}$}
\label{subsec:probing-classifier}

We will have to choose the $Lm$ probing classifiers so that $\hat{\bSigma}$ 
is well-conditioned. One way to do this is to choose  the  classifiers so that  $\hat{\bSigma}$ has a high value on the diagonal entries and a low value on the off-diagonals, i.e.\ choose each 
classifier $h^{\ell,i}$ to evaluate to a high value on $\hat{\Phi}^{\tr,\ell}_i[h]$ and a low value on $\hat{\Phi}^{\tr,\ell'}_{i'}[h], \, \forall (\ell',i') \ne (\ell,i)$. 
This can be framed as the following constraint satisfaction problem on $S^\tr$: 

\begin{center}
For $h^{\ell,i}$ pick $h \in \H$ such that:
\begin{align}
\hat{\Phi}^{\tr,\ell}_i[h] \geq \gamma,~\text{and}~
\hat{\Phi}^{\tr,\ell'}_{i'}[h] \leq \omega, \forall (\ell',i') \ne (\ell,i),
    \label{eq:con-opt}
\end{align}
\end{center}

for some $\gamma > \omega > 0$ and a sufficiently flexible hypothesis class $\H$ for which the constraints are feasible. These problems can generally be solved by formulating a constrained classification problem \cite{cotter2019optimization,narasimhan2018learning}. We show in Appendix \ref{app:con-opt} that this problem is feasible and can be efficiently solved for a range of settings.

\begin{figure}
\begin{algorithm}[H]
\caption{\hspace{-0.075cm}\textbf{: ElicitWeights} for Diagonal Linear Metrics}\label{algo:weight-coeff}
\begin{algorithmic}[1]
\STATE \textbf{Input:} $\hat{\perf}^\val$, Basis functions $\phi^1, \ldots, \phi^L: \X \> [0,1]$, 
Training set $S^\tr \sim \Dshift$, Val. set $S^\val \sim \Dtrue$, $\bar{h}$, ${\epsilon}$, $\H$, $\gamma, \omega$
\STATE \textbf{If} 
\textit{fixed classifier}:
\STATE ~~~Choose $h^{\ell,i}(x) = 
\epsilon\phi^\ell(x)\,e^i(x) + (1 - \epsilon\phi^\ell(x))\,\bar{h}(x)$ 
\STATE \textbf{Else:} 
\STATE ~~~$\bar{\H} = \{\tau h + (1-\tau)\bar{h}\,|\, h \in \H, \tau \in [0, \epsilon]\}$
\STATE ~~~Pick $h^{\ell,i} \in \bar{\H}$ to satisfy \eqref{eq:con-opt} 
with slack $\gamma,\omega, \forall (\ell,i)$
\STATE Compute $\hat{\bSigma}$ and $\hat{\bcE}$ using \eqref{eq:perf-emp} with metric $\hat{\perf}^\val$
\STATE \textbf{Output:} $\hat{\balpha} = \hat{\bSigma}^{-1}\hat{\bcE}$
\end{algorithmic}
\end{algorithm}
\vspace{-20pt}
\begin{algorithm}[H]
\caption{\hspace{-0.075cm}\textbf{:} {\textbf{P}lug-\textbf{i}n with \textbf{E}licited \textbf{W}eights} (\textbf{PI-EW}) for Diagonal Linear Metrics
}
\label{algo:linear-metrics}
\begin{algorithmic}[1]
\STATE \textbf{Input:}  $\hat{\perf}^\val$, Basis functions $\phi^1, \ldots, \phi^L: \X \> [0, 1]$, Class probability model $\hat{\eta}^\tr: \X \> \Delta_m$ for  $\Dshift$, 
Training set $S^\tr \sim \Dshift$, Validation set $S^\val \sim \Dtrue$, $\bar{h}$, $\epsilon$
\STATE $\widehat{\balpha} = \textbf{ElicitWeights}(\hat{\perf}^\val, 
\phi^1, \ldots, \phi^L, S^\tr, S^\val, \bar{h}, \epsilon)$ 
\STATE Example-weights:
$
\widehat{W}_{i}(x) \,=\, \sum_{\ell=1}^L \widehat{\alpha}^{\ell}_{i}\phi^\ell_{i}(x)
$
\STATE Plug-in:
$
\widehat{h}(x) \,\in\, \argmax_{i \in [m]} \widehat{W}_{i}(x)\hat{\eta}^\tr_i(x)
$
\STATE \textbf{Output:} $\widehat{h}$
\end{algorithmic}
\end{algorithm}
\vspace{-18pt}
\end{figure}

In practice, we do not explicitly solve \eqref{eq:con-opt} over a hypothesis class $\H$.  Instead, a simpler and surprisingly effective strategy is to 
set the probing classifiers to 
\textit{trivial} classifiers that predict the same class on all (or a subset of) examples. To build intuition for why this is a good idea, 
consider a simple setting with only one basis function $\phi^1(x) = 1, \forall x$, where the $\phi$-confusions $\hat{\Phi}^{\tr,1}_i[h] = \frac{1}{n^\tr}\sum_{(x,y)\in S^\tr}\1(y=i)h_i(x)$ are the standard confusion  entries on the training set. In this case, a trivial classifier $e^i(x) = \onehot(i), \forall x$, which predicts class $i$ on all examples, yields the highest value for $\hat{\Phi}^{\tr,1}_i$ and 0 for all other $\hat{\Phi}^{\tr,1}_j, \forall j \ne i$. In fact, in our experiments,  we set the probing classifier $h^{1,i}$ to a randomized combination of $e^i$ and some fixed base classifier $\bar{h}$:
$$h^{1,i}(x) = \epsilon e^i(x) + (1-\epsilon)\bar{h}(x),$$ for large enough $\epsilon$ so that $\hat{\bSigma}$ is well-conditioned.

Similarly, if the basis functions divide the data into $L$  clusters (as in \eqref{eq:hardlcuster}), then we can randomize between $\bar{h}$ and a {trivial} classifier that predicts a particular class $i$ on all examples assigned to the cluster $\ell \in [L]$. The confusion matrix for the resulting classifiers will have higher values than $\bar{h}$ on the $(\ell,i)$-th diagonal entry and a lower value on  other entries. These classifiers can be succinctly written as:
\begin{equation}
h^{\ell,i}(x) = \epsilon\phi^\ell(x)e^i(x) + (1-\epsilon\phi^\ell(x))\bar{h}
\label{eq:trivial-classifiers}
\end{equation}
where we again tune $\epsilon$ to make sure that the resulting $\hat{\bSigma}$ is well-conditioned. This choice of the probing classifiers also works well in practice for general basis functions $\phi^\ell$'s.

Algorithm \ref{algo:weight-coeff} summarizes the weight elicitation procedure, where  the probing classifiers are either constructed by solving the constrained satisfaction problem \eqref{eq:con-opt} or set to the ``fixed'' classifiers in \eqref{eq:trivial-classifiers}. In both cases, the algorithm takes a base classifier $\bar{h}$ and the parameter $\epsilon$ as input, where $\epsilon$ controls the extent to which $\bar{h}$ is perturbed to construct the probing classifiers. This radius parameter $\epsilon$ restricts the probing classifiers to a neighborhood around $\bar{h}$ and will prove handy in the algorithm we develop in Section~\ref{ssec:iterativeFW}.

\section{Plug-in Based Algorithms}
\label{sec:algorithms}

Having elicited the weight coefficients $\balpha$, we now seek to learn a classifier that optimizes the left hand side of~\eqref{eq:example-weights-reduced}. 
We do this via the \textit{plug-in} approach: first \textit{pre-train} a model $\hat{\eta}^\tr: \X \> \Delta_m$  on the noisy training distribution $\Dshift$ to estimate the conditional class probabilities $\hat{\eta}^\tr_i(x) \approx \P^\mu(y=i|x)$, and then apply the correction weights to \emph{post-shift} $\hat{\eta}^\tr$. 

\begin{figure}
\begin{algorithm}[H]
\caption{\hspace{-0.075cm}\textbf{:} \textbf{F}rank-\textbf{W}olfe with \textbf{E}licited \textbf{G}radients (\textbf{FW-EG}) for General Diagonal Metrics (also depicted in Fig.~\ref{fig:overview})}\label{algo:FW}
\begin{algorithmic}[1]
\STATE \textbf{Input:} $\hat{\perf}^\val$, Basis functions $\phi^1, \ldots, \phi^L: \X \> [0,1]$, Pre-trained $\hat{\eta}^\tr: \X \> \Delta_m$,
$S^\tr \sim \Dshift$, 
$S^\val \sim \Dtrue$, $T$, $\epsilon$
\STATE Initialize classifier $h^0$ and $\c^0 = \diag(\widehat{\C}^\val[h^0])$
\STATE \textbf{For} $t =  0$ \textbf{to} $T-1$ \textbf{do}
\STATE ~~~\textbf{if} $\perf^\Dtrue[h] = \psi(C_{11}^D[h],\ldots,C_{mm}^D[h])$ for known $\psi$:
\STATE ~~~~~~~$\bbeta^{t}\,=\, \nabla\psi(\c^{t})$
\STATE ~~~~~~~$\hat{\perf}^\lin[h]\,=\, \sum_i \beta^t_i \hat{C}^\val_{ii}[h]$
\STATE ~~~\textbf{else}
\STATE ~~~~~~~$\hat{\perf}^\lin[h]= \hat{\perf}^\val[h]$ \hspace{0.7cm}\COMMENT{small $\epsilon$ recommendeded}
\STATE ~~~$\widehat{f} = \text{\textbf{PI-EW}}(\hat{\perf}^\lin, \phi^1,..., \phi^L, \hat{\eta}^\tr, S^\tr, S^\val, h^{t}, \epsilon)$
\STATE ~~~$\tilde{\c} = \diag(\widehat{\C}^\val[\widehat{f}])$
\STATE ~~~${h}^{t+1} = \big(1-\frac{2}{t+1}\big) {h}^{t} + \frac{2}{t+1} \onehot(\widehat{f})$
\STATE ~~~${\c}^{t+1} = \big(1-\frac{2}{t+1}\big) {\c}^{t} + \frac{2}{t+1}\tilde{\c}$
\STATE \textbf{End For}
\STATE \textbf{Output:} $\hat{h} = h^T$
\end{algorithmic}
\end{algorithm}
\vspace{-18pt}
\end{figure}

\subsection{Plug-in  Algorithm for Linear Metrics}
We first describe our approach for (diagonal) linear metrics $\perf^\Dtrue[h] \,=\,  \sum_{i}\beta_i\,C^D_{ii}[h]$ 
in Algorithm \ref{algo:linear-metrics}.
Given the correction weights $\hat{\W}: \X \> \R_+^m$, we seek to maximize the following weighted objective on the training distribution: 
\[
\textstyle
\max_{h}\,\E_{(x, y) \sim \Dshift}\left[\sum_{i} \hat{W}_{i}(x)\,\1(y = i)h_i(x)\right].
\]
This is a standard example-weighted learning problem, for which the following plug-in (post-shift) classifier is a consistent estimator \cite{narasimhan2015consistent,yang2020fairness}:
\[
\widehat{h}(x) \,\in\, \argmax_{i \in [m]} \hat{W}_{i}(x)\,\hat{\eta}^\tr_i(x).
\]

\subsection{Iterative Algorithm for General Metrics}
\label{ssec:iterativeFW}

To optimize generic non-linear metrics of the form $\perf^\Dtrue[h] = \psi(C^D_{11}[h], \ldots,C^D_{mm}[h])$ for $\psi: [0,1]^m\>\R_+$, we  apply Algorithm \ref{algo:linear-metrics} iteratively. We consider both cases where $\psi$ is unknown, and where $\psi$ is known, but needs to be optimized using the noisy distribution $\Dshift$. The idea is to first elicit local linear approximations to $\psi$ and to then learn plug-in classifiers for the resulting linear metrics in each iteration. 

Specifically, following~\citet{narasimhan2015consistent}, we derive our algorithm from the classical Frank-Wolfe  method \cite{Jaggi13} for maximizing a smooth concave function $\psi(\c)$ over a convex set $\cC \subseteq \R^m$. In our case, $\cC$ is the set of confusion matrices $\C^\Dtrue[h]$ achieved by any classifier $h$, and is convex when we allow randomized classifiers (see Lemma~\ref{lem:C-convexity}, Appendix \ref{app:proof-FW}). The algorithm maintains iterates $\c^t$, and at each step, maximizes a linear approximation to $\psi$ at $\c^t$: $\tilde{\c} \,\in\, \argmax_{\c \in \cC}\langle \nabla\psi(\c^{t}), \c\rangle$. The next iterate $\c^{t+1}$ is then a convex combination of $\c^{t}$ and the current solution $\tilde{\c}$.

In Algorithm \ref{algo:FW}, we outline an adaptation of this Frank-Wolfe algorithm to our setting, where we maintain a classifier $h^t$ and an estimate of the diagonal confusion entries $\c^t$ from the validation sample $S^\val$. At each step, we linearize $\psi$ using $\hat{\perf}^\lin[h] = \sum_{i} \beta^t_i \hat{C}_{ii}^\val[h]$, where $\bbeta^t = \nabla \psi(\c^t)$, and invoke the plug-in method in Algorithm \ref{algo:linear-metrics} to optimize the linear approximation $\hat{\perf}^\lin$. 
When the mathematical form of $\psi$ is known, one can directly compute the gradient $\bbeta^t$. When it is not known, we can simply set $\hat{\perf}^\lin[h] = \hat{\perf}^\val[h]$, but restrict the  weight elicitation routine (Algorithm \ref{algo:weight-coeff}) to choose its probing classifiers $h^{\ell,i}$'s from a small neighborhood  around the current classifier $h^t$ (in which $\psi$ is effectively linear). This can be done by passing $\bar{h}=h^t$ to the weight elicitation routine, and setting the radius ${\epsilon}$  to a small value. 
    
Each call to Algorithm \ref{algo:linear-metrics} uses the training and validation set
to elicit example weights for a local linear approximation to $\psi$, and uses the  weights to 
construct a plug-in classifier. The final  output is a randomized combination of the plug-in classifiers  from each step. Note that Algorithm~\ref{algo:FW} runs efficiently for reasonable values of $L$ and $m$. Indeed the runtime is almost always dominated by the pre-training of the base model $\hat\eta^\tr$, with the time taken to elicit the weights (e.g.\ using~\eqref{eq:trivial-classifiers}) being relatively inexpensive (see App.\ \ref{app:run-time}).

%% file: theory.tex
\section{Theoretical Guarantees}
\label{sec:theory}

We provide theoretical guarantees for the weight elicitation procedure and the plug-in methods in Algorithms \ref{algo:weight-coeff}--\ref{algo:FW}.

\begin{asp}
\label{asp:alpha-star}
The distributions $\Dtrue$ and $\Dshift$ are such that
for any linear metric $\perf^\Dtrue[h] = \sum_{i}\beta_i C_{ii}[h]$, with $\|\bbeta\| \leq 1$, 
$\exists \bar{\balpha} \in \R^{Lm}$ s.t. $\left|\sum_{\ell,i}
\bar{\alpha}^{\ell}_i
\Phi^{\Dshift, \ell}_i[h] -
\perf^\Dtrue[h]\right| \,\leq\, \nu, \forall h$ and $\|\bar{\balpha}\|_1 \leq B$, for some $\nu \in [0,1)$ and $B>0$.
\end{asp}

The assumption states that our choice of basis functions $\phi^1,\ldots,\phi^L$ are such that, any linear metric on  $\Dtrue$ can be approximated (up to a slack $\nu$) by a weighting $W_i(x) = \sum_{\ell}\bar{\alpha}^\ell_i\phi^\ell(x)$ of the  training examples  from $\Dshift$.
The existence of such a weighting function depends on how well the basis functions capture the underlying distribution shift. 
Indeed, the assumption holds for some common settings in Table \ref{tab:correction-weights}, e.g., when the noise  transition $\T$ is diagonal (Appendix \ref{app:linear-gen} handles a general $\T$), and the basis functions are set to $\phi^1(x)=1,\forall x,$ for the IDLN setting, and $\phi^\ell(x) = \1(g(x)=\ell), \forall x,$ for the CDLN setting.

We  analyze the coefficients $\hat{\balpha}$ elicited by Algorithm \ref{algo:weight-coeff} when the probing classifiers $h^{\ell,i}$ are chosen to satisfy \eqref{eq:con-opt}. In Appendix  \ref{app:norm-sigma-bound}, we provide an analysis  when the probing classifiers $h^{\ell,i}$ are set to the fixed choices in \eqref{eq:trivial-classifiers}. 
\begin{thm}[\hspace{-1pt}\textbf{Error bound on elicited weights}]
\label{thm:alpha-diagonal-linear-conopt}
Let $\gamma, \omega > 0$ be such that the constraints in \eqref{eq:con-opt}  are feasible for hypothesis class $\bar{\H}$, for all $\ell, i$. Suppose Algorithm \ref{algo:weight-coeff} chooses each
 classifier $h^{\ell,i}$ to  satisfy \eqref{eq:con-opt}, with $\perf^D[h^{\ell,i}] \in [c, 1], \forall \ell, i$, for some $c>0$. Let $\bar{\alpha}$ be defined as in Assumption \ref{asp:alpha-star}. Suppose $\gamma > 2\sqrt{2}Lm\omega$ and $n^\tr \geq \frac{L^2m\log(Lm|\H|/\delta)}{(\frac{\gamma}{2} - \sqrt{2}Lm\omega)^2}.$
Fix $\delta\in (0,1)$. Then w.p.\  $\geq 1 - \delta$ over draws of $S^\tr$ and $S^\val$ from $\Dshift$ and $\Dtrue$ resp., the coefficients $\hat{\balpha}$ output by Algorithm \ref{algo:weight-coeff} satisfies:
\begin{eqnarray*}
\|\hat{\balpha} - \bar{\balpha}\| \,\leq\,
\mathcal{O}\bigg(
\frac{Lm}{\gamma^2}\bigg(\sqrt{\frac{L\log(\textstyle \frac{Lm|\H|}{\delta})}{n^\tr}} + 
 \sqrt{\frac{L\log(\textstyle \frac{Lm}{\delta})}{c^2 n^\val}}\bigg) +  \frac{\nu\sqrt{Lm}}{\gamma}\bigg),
\end{eqnarray*}
where the term $|\H|$ can be replaced by a measure of capacity of the hypothesis class $\H$.
\end{thm}

Because the probing classifiers are chosen using the training set alone, it is only the sampling errors from the training set that depend on the complexity of $\H$, and not those from the validation set. 
This suggests robustness of our approach to  a small validation set as long as the training set is sufficiently large and the number of basis functions is reasonably small. 

For the iterative plug-in method in Algorithm \ref{algo:FW}, we bound the gap between the metric value $\perf^D[\hat{h}]$ for the output classifier $\hat{h}$ on the true distribution $\Dtrue$, and the optimal value. We handle the case where the function $\psi$ is \textit{known} and its gradient $\nabla\psi$ can be computed in closed-form. The more general case of an unknown $\psi$ is handled in Appendix \ref{app:complex-unknown}.
The above bound depends on the gap between the estimated  class probabilities $\hat{\eta}_i^\tr(x)$ for the training distribution and true class probabilities $\eta_i^\tr(x) = \P(y=i|x)$, as well as the quality of the coefficients $\hat{\balpha}$ provided by the weight estimation subroutine, as measured by $\kappa(\cdot)$. 
One can substitute $\kappa(\cdot)$ with, e.g., the error bound provided in Theorem \ref{thm:alpha-diagonal-linear-conopt}. 

\begin{thm}[\textbf{Error Bound for FW-EG}]
\label{thm:iterative-plugin}
Let $\perf^\Dtrue[h] = \psi(C^\Dtrue_{11}[h],\ldots, C^\Dtrue_{mm}[h])$ for a \emph{known} concave function $\psi: [0,1]^m \>\R_+$, which is $Q$-Lipschitz and $\lambda$-smooth.  Fix $\delta \in (0, 1)$.
Suppose Assumption \ref{asp:alpha-star} holds, and for any linear metric $\sum_i\beta_i C^D_{ii}[h]$, whose associated weight coefficients is $\bar{\balpha}$ with $\|\bar{\balpha}\| \leq B$, 
 w.p. $\geq 1-\delta$ over draw of $S^\tr$ and $S^\val$, the weight estimation routine in Alg.\ \ref{algo:weight-coeff} outputs coefficients $\hat{\balpha}$ with
 $\|\hat{\balpha} -\bar{\balpha}\| \leq 
 \kappa(\delta, n^\tr, n^\val)
 $, for some function $\kappa(\cdot) > 0$. Let $B' = B + \sqrt{Lm}\,\kappa(\delta/T, n^\tr, n^\val).$ Then w.p.\  $\geq 1 - \delta$ over draws of $S^\tr$ and $S^\val$ from $\Dtrue$ and $\Dshift$ resp., the classifier $\hat{h}$ output by Algorithm \ref{algo:FW} after $T$ iterations satisfies:
\begin{eqnarray*}
\max_{h}\perf^D[h] - \perf^D[\hat{h}]\,\leq&
2QB'\E_x\left[\|\eta^\tr(x)-
\hat{\eta}^\tr(x)\|_1\right] +
 4Q\sqrt{Lm}\,\kappa(\textstyle\frac{\delta}{T}, n^\tr, n^\val) +\\
 &
 \hspace{-1.5cm}
\mathcal{O}\left(\lambda m\sqrt{\frac{m\log(m)\log(n^\val) + \log(m/\delta)}{n^\val}} + \frac{\lambda}{T} + Q\nu\right).
\end{eqnarray*}
\vskip -0.2cm
\end{thm}
\vspace{-10pt}

The proof in turn derives an error bound for the plug-in classifier in Algorithm \ref{algo:linear-metrics} for linear metrics (see App.\ \ref{app:pi-ew}). 

%% file: relatedwork.tex
\section{Related Work}
\label{sec:relatedwork}

\textbf{Methods for closed-form metrics.} There has been a variety of work on optimizing complex evaluation metrics, including both plug-in type algorithms~\cite{ye2012optimizing, narasimhan2014statistical, koyejo2014consistent, narasimhan2015consistent, yan2018binary}, and those that use convex surrogates for the metric~\cite{joachims2005support, kar2014online, kar2016online, narasimhan2015optimizing,eban2017scalable, narasimhan2019optimizing,hiranandani2020optimization}. These methods rely on the test metric having a specific closed-form structure and do not handle black-box metrics.

\textbf{Methods for black-box metrics.} 
Among recent black-box metric learning works, the closest to ours is \citeauthor{jiang2020optimizing}~\citeyearpar{jiang2020optimizing}, who  learn a weighted combination of surrogate losses to approximate the metric on a validation set.  
Like us, they probe the metric at multiple classifiers, but their approach has several drawbacks on both practical and theoretical fronts. Firstly, \citet{jiang2020optimizing} require retraining the model in each iteration, which can be time-intensive, whereas we only post-shift a pre-trained model. Secondly, the procedure they
prescribe for eliciting gradients requires perturbing the model parameters multiple times, which can be very expensive for large deep networks, whereas we only require perturbing the predictions from the model. Moreover, the number of perturbations they need grows \textit{polynomially} with the precision with which they need to estimate the loss coefficients, whereas we only require a \textit{constant} number of them.
Lastly, their approach does not come with strong statistical guarantees, whereas  ours does. Besides these benefits over~\cite{jiang2020optimizing}, we will also see in Section~\ref{sec:experiments} that our method yields better accuracies. Other related black-box learning methods include~\citet{zhao2019metric}, \citet{ren2018learning}, and~\citet{huang2019addressing}, who learn a (weighted) loss  to approximate the metric,
but do so using computationally expensive procedures (e.g.\ meta-gradient descent or RL) that often require retraining the model from scratch, and come with limited theoretical analysis. 

\textbf{Methods for distribution shift.} The literature on distribution shift is vast, and so we cover a few representative papers; see~\cite{frenay2013classification,csurka2017comprehensive} for a comprehensive discussion. For the \textit{independent label noise} setting~\cite{natarajan2013learning}, ~\citeauthor{patrini2017making}~\citeyearpar{patrini2017making} propose a loss correction approach that first trains a model with noisy label, use its predictions to estimate the noise transition matrix, and then re-trains model with the corrected loss. This approach is however tailored to optimize linear metrics; whereas, we can handle more complex metrics as well without re-training the underlying model. A plethora of  approaches exist for tackling \textit{domain shift}, including classical importance weighting (IW) strategies~\cite{sugiyama2008direct,shimodaira2000improving, kanamori2009least, lipton2018detecting} that work in two steps:  estimate the density ratios and train a model with the resulting weighted loss. One such approach is Kernel Mean Matching~\cite{huang2006correcting}, which matches covariate distributions between training and test sets in a high dimensional RKHS feature space. These IW approaches are however prone to over-fitting when used with deep networks~\cite{byrd2019effect}. More recent iterative variants seek to remedy this~\cite{fang2020rethinking}.

%% file: experiments.tex
\vspace{-0.25cm}
\section{Experiments}
\label{sec:experiments}
\vskip -0.15cm
We run experiments on four classification tasks, with both known and black-box metrics, and under different label noise and domain shift settings. 
All our experiments use a large training sample, which is either noisy or contains missing attributes, and a smaller clean (and complete) validation sample. We always optimize the cross-entropy loss for learning  $\hat\eta^{\tr}(x) \approx \P^\Dshift(Y|x)$ using the training set (or $\hat\eta^{\val}(x) \approx \P^\Dtrue(Y|x)$ for some baselines), where the  models are varied across experiments. For monitoring the quality of $\hat\eta^{\tr}$ and $\hat\eta^{\val}$, we sample small subsets \emph{hyper-train} and \emph{hyper-val} data from the original training and validation data, respectively. We repeat our experiments over 5 random train-vali-test splits, and report the mean and standard deviation for each metric. We will use $^*$, $^{**}$, and $^{***}$ to denote that the differences between our method and the closest baseline are statistically significant (using Welch's t-test)  at a confidence level of 90\%, 95\%, and 99\%, respectively. Table~\ref{tab:stats} in App.~\ref{app:exp} summarizes the datasets used. The source code (along with random seeds) is provided on the link below.\footnote{\url{https://github.com/koyejolab/fweg/}}

\textbf{Common baselines}:
We use representative baselines from the black-box learning \cite{jiang2020optimizing},  
iterative re-weighting \cite{ren2018learning}, label noise correction \cite{patrini2017making}, and importance weighting \cite{huang2006correcting} literatures.
First, we list the ones common to all experiments.

\begin{enumerate}[leftmargin=1cm, itemsep=-2pt]
    \item \textbf{Cross-entropy [train]:} Maximizes accuracy on the training set and predicts:
    $$
    \widehat{h}(x) \,\in\, \argmax_{i \in [m]} \hat{\eta}^\tr_i(x).
    $$
    \item \textbf{Cross-entropy [val]:} 
    Maximizes accuracy on the validation set and predicts:
    $$
    \widehat{h}(x) \,\in\, \argmax_{i \in [m]} \hat{\eta}^\val_i(x).
    $$
    \item \textbf{Fine-tuning:} Fine-tunes the pre-trained $\hat\eta^{\text{tr}}$ using the validation data, monitoring the cross-entropy loss on the hyper-val data for early stopping. 
    \item \textbf{Opt-metric [val]:} 
    For  metrics $\psi(\C^D[h])$, for which $\psi$ is  \textit{known}, trains a model to directly maximize the metric on the small \textit{validation} set using the Frank-Wolfe based algorithm of \cite{narasimhan2015consistent}.
    \item \textbf{Learn-to-reweight} \cite{ren2018learning}:  Jointly  learns example weights, with the model,  to maximize accuracy on the validation set; does not handle specialized metrics.
    \item \textbf{Plug-in [train-val]:} 
    Constructs a classifier $\widehat{h}(x) \,\in\, \argmax_{i} w_i\hat{\eta}^\val_i(x)$, where the weights $w_i \in \R$ are tuned to maximize the given metric on the validation set, using a coordinate-wise line search (details in  Appendix~\ref{ssec:multiclass-plugin}). 
    \item \textbf{Adaptive Surrogates}~\cite{jiang2020optimizing}: Learns a weighted combination of surrogate losses (evaluated on clusters of examples) 
    to approximate the metric on the validation set. Since this method is not directly amenable for use with large neural networks (see Section~\ref{sec:relatedwork}), we compare with it only when using linear models, and present additional comparisons in App.~\ref{app:exp} (Table~\ref{tab:addedexp}).
\end{enumerate}

\textbf{Hyper-parameters:} The learning rate for Fine-tuning is chosen from  $1\text{e}^{\{-6,\dots,-4\}}$. For PI-EW and FW-EG, we tune the parameter $\epsilon$ from $\{1, 0.4, 1\text{e}^{-\{4,3,2,1\}}\}$. The line search for Plug-in is performed with a  spacing of $1\text{e}^{-4}$. The only hyper-parameters the other baselines have are those for training $\hat{\eta}^\tr$ and $\hat{\eta}^\val$, which we state in the individual tasks.

\subsection{Maximizing Accuracy under Label Noise}
\label{ssec:cifar10}

\begin{table}[t]
    \small
    \centering
    \caption{Test accuracy for noisy label experiment on CIFAR-10.}
    \vskip -0.25cm
    \begin{tabular}{ll}
    \hline
        Cross-entropy [train] &  0.582 $\pm$ 0.007\\
        Cross-entropy [val] & 0.386 $\pm$ 0.031 \\
        Learn-to-reweight & 0.651	$\pm$ 0.017\\
        Plug-in [train-val] & 0.733	$\pm$ 0.044\\
        Forward Correction & 0.757 $\pm$	0.005\\
        Fine-tuning & 0.769 $\pm$	0.005\\
        \hline
        PI-EW & $\textbf{0.781} \pm \textbf{0.019}$\\
        \hline
    \end{tabular}
    \label{tab:cifar10assym}
    \vskip -0.3cm
\end{table}

In our first task, we train a 10-class image classifier for the CIFAR-10 dataset~\cite{krizhevsky2009learning}, replicating the independent (asymmetric) label noise setup
from~\cite{patrini2017making}. The evaluation metric we use is accuracy. We take 2\% of original training data as validation data and flip labels in the remaining training set based on the following transition matrix: {\small TRUCK $\rightarrow$ AUTOMOBILE, BIRD $\rightarrow$ PLANE, DEER $\rightarrow$ HORSE, CAT $\leftrightarrow$ DOG}, with a flip probability of 0.6. For $\hat\eta^{\text{tr}}$ and $\hat\eta^{\text{val}}$, we use the same ResNet-14 architecture as~\cite{patrini2017making}, trained using SGD with momentum 0.9, weight decay $1\text{e}^{-4}$, 
and  learning rate 0.01, which we divide by 10 after 40 and 80 epochs (120 in total). 

We additionally compare with the \emph{Forward Correction} method of~\cite{patrini2017making}, a specialized method for correcting independent label noise, which estimates the noise transition matrix $\T$ using predictions from $\hat\eta^{\text{tr}}$ on the training set, and retrains it with the corrected loss, thus  training the ResNet twice. 
We saw a notable drop with this method when we used the (small) validation set to estimate $\T$.

We apply the proposed PI-EW method for linear metrics, using a weighting function $\W$ defined with one of two choices for the basis functions (chosen via cross-validation): (i) a default basis function that clusters all the points together $\phi^{\text{def}}(x) = 1 \, \forall x$, and (ii) 
ten basis functions $\phi^1, \dots, \phi^{10}$, each one being the average of the RBF kernels (see \eqref{eq:softlcuster}) centered at validation points belonging to a true class. The RBF kernels are computed with width 2 on UMAP-reduced 50-dimensional image embeddings~\cite{mcinnes2018umap}.

As shown in Table~\ref{tab:cifar10assym}, PI-EW achieves significantly better test accuracies than all the baselines. The results for Forward Correction matches those in \cite{patrini2017making}; unlike this method, we train the ResNet only once, but achieve  2.4\% higher accuracy. Cross-entropy [val] over-fits badly, and yields the least test accuracy. Surprisingly, the simple fine-tuning yields the second-best accuracy. A possible reason is that the pre-trained model learns a good feature representation, and the fine-tuning step adapts well to the domain change. We also observed that PI-EW achieves better accuracy during cross-validation with ten basis functions,  highlighting the benefit of the underlying modeling in PI-EW. Lastly, in Figure~\ref{fig:weightscifar}, we show the elicited (class) weights with the default basis function  ($\phi^{\text{def}}(x) = 1 \, \forall x$), where e.g. because BIRD $\rightarrow$ PLANE, the weight on BIRD is upweighted and that on PLANE is down-weighted. 

\subsection{Maximizing G-mean with Proxy Labels}
\label{ssec:adult}

Our next experiment borrows the  ``proxy label'' setup from ~\citet{jiang2020optimizing} on the Adult dataset~\cite{Dua:2019}. The task is to predict whether a candidate's gender is male, but the training set contains only a proxy for the true label. 
We  sample 1\% validation data from the original training data, and replace the labels in the remaining sample with the feature `relationship-husband'. The label noise here is instance-dependent (see Example~\ref{ex:instnoise}), and we seek to maximize the G-mean metric: {\small $\psi(\C) \,=\,\big(\prod_i \big({C_{ii}}/\sum_j C_{ij}\big)\big)^{1/m}$}. 

We train $\hat\eta^{\text{tr}}$ and $\hat\eta^{\text{val}}$ using linear logistic regression using SGD with a learning rate of 0.01. As additional baselines, we include the Adaptive Surrogates method of \cite{jiang2020optimizing} and \emph{Forward Correction}~\cite{patrini2017making}. The inner and outer learning rates for Adaptive Surrogates are each cross-validated in $\{0.1, 1.0\}$. We also compare with a simple Importance Weighting strategy, where we first train a  logistic regression model $f$ to predict if an example $(x,y)$ belongs to the validation data, and train a gender classifier with the training examples weighted by {\small $f(x,y)/(1 - f(x,y))$}. 

We choose between three sets of basis functions (using cross-validation): (i) a default basis function 
$\phi^{\text{def}}(x) = 1 \, \forall x$, (ii) $\phi^{\text{def}}, \phi^{\text{pw}}, \phi^{\text{npw}}$, where  $\phi^{\text{pw}}(x) = \1(x_{\text{pw}} = 1)$ and $\phi^{\text{npw}}(x) = \1(x_{\text{npw}} = 1)$ use features  `private-workforce' and `non-private-workforce' to form hard clusters, (iii) $\phi^{\text{def}}$, $\phi^{\text{pw}}, \phi^{\text{npw}}, \phi^{\text{inc}}$, where $\phi^{\text{inc}}(x) = \1(x_{\text{inc}} = 1)$ uses the binary feature `income'. These choices are motivated from those used by~\cite{jiang2020optimizing}, who compute surrogate losses on the individual clusters. We provide their Adaptive Surrogates method with the same clustering choices.

Table~\ref{tab:adultproxy} summarizes our results. We apply both variants of our FW-EG method for a non-linear metric $\psi$, one where $\psi$ is \textit{known} and its gradient is available in closed-form, and the other where $\psi$ is assumed to be \textit{unknown}, and is treated as a general black-box metric. Both variants perform similarly and are better than the baselines. Adaptive Surrogates comes a close second, but underperforms by 0.3\% (with results being statistically significant). While the improvement of FW-EG over Adaptive Surrogates is small, the latter is time intensive as, in each iteration, it re-trains a logistic regression model. We verify this empirically in Figure~\ref{fig:time} by reporting run-times for Adaptive Surrogates and our method FW-EG (including the pre-training time) against the choices of basis functions (clustering features). We see that our approach is 5$\times$ faster for this experiment. Lastly, Forward Correction performs poorly, likely because its loss correction is not aligned with this label noise model.

\begin{table}[t]
    \small
    \centering
    \caption{Test G-mean for proxy label experiment on Adult.}
    \vskip -0.25cm
    \begin{tabular}{ll}
    \hline
        Cross-entropy [train] &  0.654	$\pm$ 0.002\\
        Cross-entropy [val] & 0.394 $\pm$	0.064 \\
        Opt-metric [val] & 0.652 $\pm$	0.027\\
        Learn-to-reweight & 0.668	$\pm$ 0.003\\
        Plug-in [train-val] & 0.672 $\pm$ 	0.013\\
        Forward Correction & 0.214 $\pm$	0.004\\
        Fine-tuning & 0.631	$\pm$ 0.017\\
        Importance Weights & 0.662 $\pm$	0.024\\
        Adaptive Surrogates & 0.682	 $\pm$ 0.002\\
        \hline
        FW-EG [unknown $\psi$] & $\textbf{0.685}	\pm \textbf{0.002}^{**}$\\
        FW-EG [known $\psi$] & $\textbf{0.685}	\pm \textbf{0.001}^{*}$\\
        \hline
    \end{tabular}
    \label{tab:adultproxy}
    \vskip -0.3cm
\end{table}

\subsection{Maximizing F-measure under Domain Shift}
\label{ssec:adience}

We now move on to a domain shift application  (see Example~\ref{ex:ds}). The task is to learn a gender recognizer for the Adience face image dataset~\cite{eidinger2014age}, but with the training and test datasets containing images from different age groups (domain shift based on age). We use images belonging to age buckets 1--5 for training (12.2K images), and evaluate on images from age buckets 6--8 (4K images). For the validation set, we sample 20\% of the 6--8 age bucket images. Here we aim to maximize the F-measure.

For $\hat\eta^{\text{tr}}$ and $\hat\eta^{\text{val}}$, we
use the same ResNet-14 model from the CIFAR-10 experiment, except that the learning rate is divided by 2 after 10 epochs (20 in total). As an additional baseline, we compute importance weights using \emph{Kernel Mean Matching (KMM)}~\cite{huang2006correcting}, and train the same ResNet model with a weighted loss. Since the image size is large for directly applying KMM, we first compute the 2048-dimensional ImageNet embedding~\cite{krizhevsky2012imagenet} for the images and further reduce them to 10-dimensions via UMAP. The KMM weights are learned on the 10-dimensional embedding. For the basis functions, besides the default basis $\phi^{\text{def}}(x) = 1 \, \forall x$, we choose from subsets of six RBF basis functions $\phi^1,\ldots,\phi^6$, centered at points from the validation set, each representing one of six age-gender combinations. We use the same UMAP embedding as KMM to compute the RBF kernels.

Table~\ref{tab:adiencecovs} presents the test F-measure values. Both variants of FW-EG algorithm provide statistically significant improvements over the baselines. Both Fine-tuning and Learning-to-reweight improve over plain cross-entropy optimization (train), however only moderately, likely because of the small size of the validation  set, and because these methods are not tailored to optimize the F-measure. 

\begin{table}[t]
    \small
    \centering
    \caption{Test F-measure for domain shift experiment on Adience.}
    \vskip -0.25cm
    \begin{tabular}{ll}
    \hline
        Cross-entropy [train] & 0.760 $\pm$	0.014\\
        Cross-entropy [val] & 0.708 $\pm$	0.022 \\
        Opt-metric [val] & 0.760 $\pm$	0.014\\
        Plug-in [train-val] & 0.759 $\pm$	0.014\\
        Importance Weights [KMM] & 0.760 $\pm$	0.013\\
        Learn-to-reweight & 0.773	$\pm$ 0.009\\
        Fine-tuning & 0.781	$\pm$ 0.014\\
        \hline
        FW-EG [unknown $\psi$] & $\textbf{0.815}	\pm \textbf{0.013} ^{***}$\\
        FW-EG [known $\psi$] & $\textbf{0.804} \pm	\textbf{0.015}^{***}$\\
        \hline
    \end{tabular}
    \label{tab:adiencecovs}
    \vskip -0.3cm
\end{table}

\subsection{Maximizing Black-box Fairness Metric}
\label{ssec:adultbb}

We next handle a black-box metric given only query access to its value. We consider a fairness application where the goal is to balance classification performance across multiple protected groups.  The groups that one cares about are known,
but due to privacy or legal restrictions, the protected attribute for an individual cannot be  revealed~\cite{awasthi+21}. Instead, we have access to an oracle that reveals the value of the fairness metric for predictions on a validation sample, with the protected attributes absent from the training sample. This setup is different from recent work on learning fair classifiers from incomplete group information \cite{lahoti2019ifair, wang2020robust}, in that the focus here is on optimizing \textit{any} given black-box fairness metric.

We use the Adult dataset, and seek to predict whether the candidate's income is greater than \$50K, with \textit{gender} as the protected group. The black-box  metric we consider (whose form is unknown to the learner) is the geometric mean of the true-positive (TP) and true-negative (TN) rates, evaluated separately on the male and female examples, which promotes equal performance for both groups and classes:
\begin{equation*}
\perf^\Dtrue[h] = \left( \text{TP}^{\text{male}}[h]\,\text{TN}^{\text{male}}[h])\,\text{TP}^{\text{female}}[h]\,\text{TN}^{\text{female}}[h] \right)^{1/4}.
\label{eq:fairmetric}
\end{equation*}

We train the same logistic regression models as in previous Adult experiment in Section \ref{ssec:adult}. Along with the  basis functions $\phi^{\text{def}}$, $\phi^{\text{pw}}$ and $\phi^{\text{npw}}$ we used there, we additionally include two basis $\phi^{\text{hs}}$ and $\phi^{\text{wf}}$  based on features `relationship-husband' and `relationship-wife', which we expect to have correlations with gender.\footnote{The only domain knowledge we use is that the protected group is ``gender"; beyond this, the
form of the metric is unknown, and importantly, an individual's gender is not available.} 
We include two baselines that can handle black-box metrics: Plug-in [train-val], which tunes a threshold on $\hat{\eta}^\tr$ by querying the metric on the validation set, and Adaptive Surrogates. The latter is cross-validated on the same set of clustering features (i.e., basis functions in our method) for computing the surrogate losses. 

As seen in Table~\ref{tab:adultbb}, FW-EG yields the highest black-box metric on the test set, Adaptive Surrogates comes in second, and surprisingly the simple plug-in approach fairs better than the other baselines. During cross-validation, we also observed that the performance of FW-EG improves with more basis functions, particularly with the ones that are better correlated with gender. Specifically, 
FW-EG with basis functions  $\{\phi^{\text{def}}, \phi^{\text{pw}}, \phi^\text{npw},  \phi^\text{wf}, \phi^\text{hs}\}$ achieves
approximately
1\% better performance than both
FW-EG with $\phi^{\text{def}}$ basis function and FW-EG with basis functions $\{\phi^{\text{def}},\phi^{\text{pw}}, \phi^\text{npw}\}$.

\subsection{Ablation Studies}
\label{ssec:ablation}

\begin{figure*}[t]
	\centering 
	\subfigure[]{
        \hspace{-0.3cm}
		{\includegraphics[width=4cm]{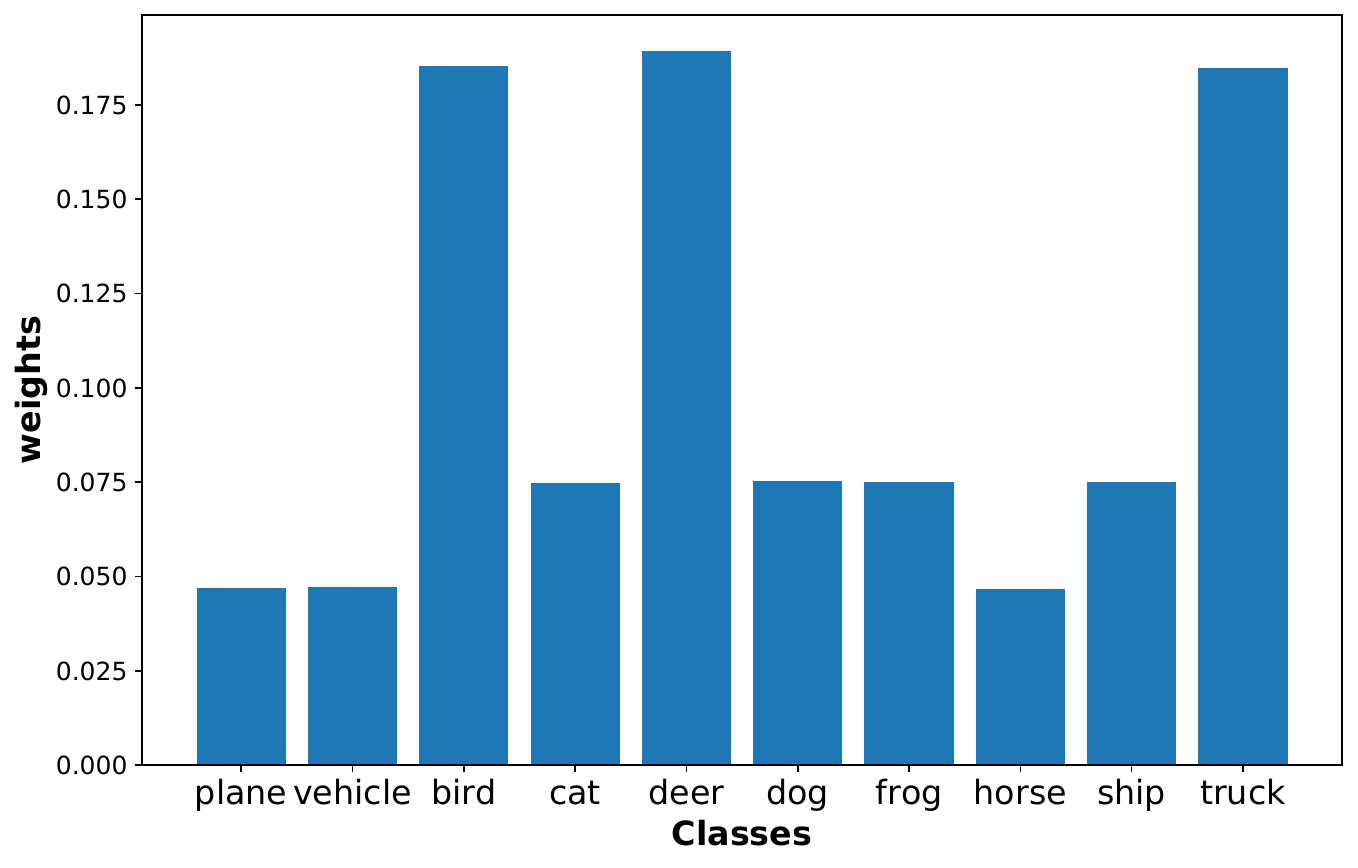}}
		\label{fig:weightscifar}
	}
	\subfigure[]{
        \hspace{-0.375cm}
		{\includegraphics[width=4cm]{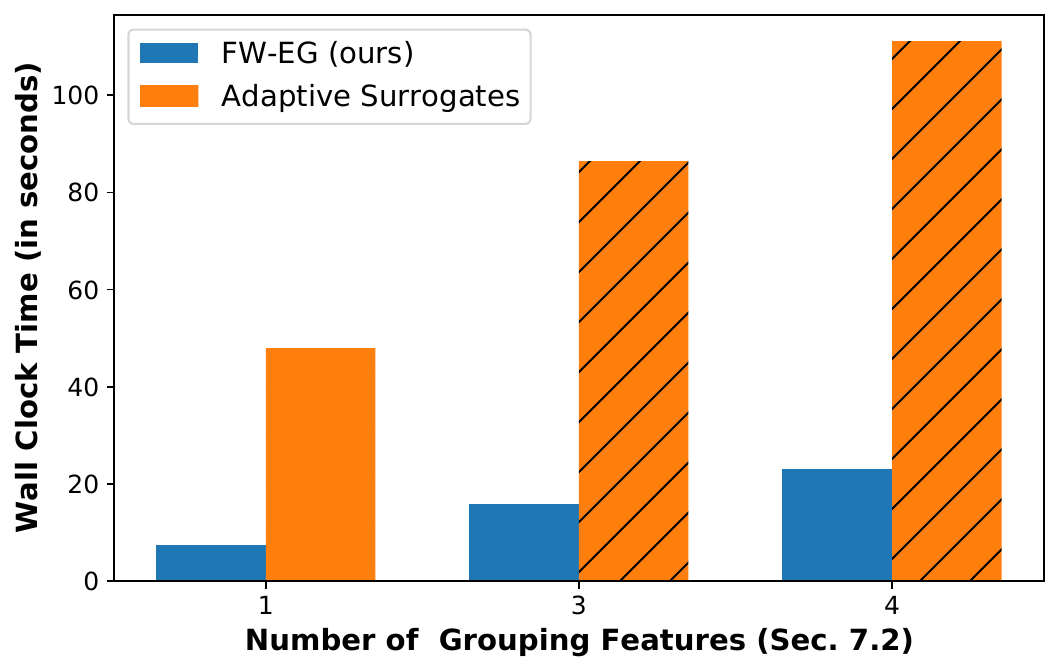}}
		\label{fig:time}
	}
	\subfigure[]{
        \hspace{-0.375cm}
		{\includegraphics[width=4cm]{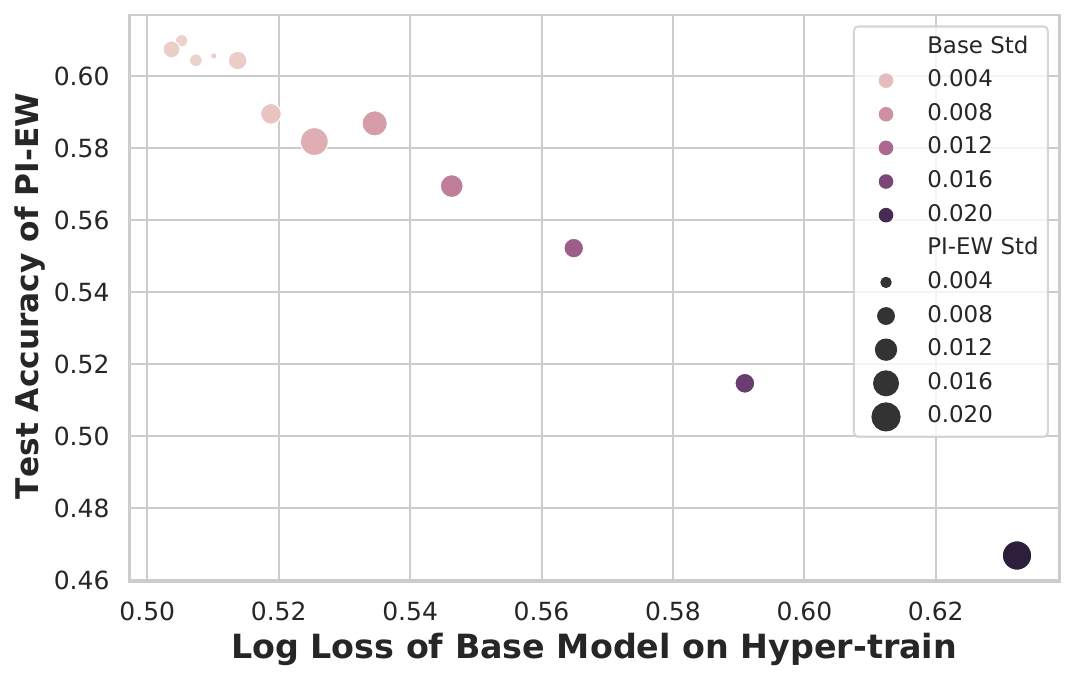}}
		\label{fig:periodicmain}
	}
	\subfigure[]{
        \hspace{-0.375cm}
		{\includegraphics[width=4cm]{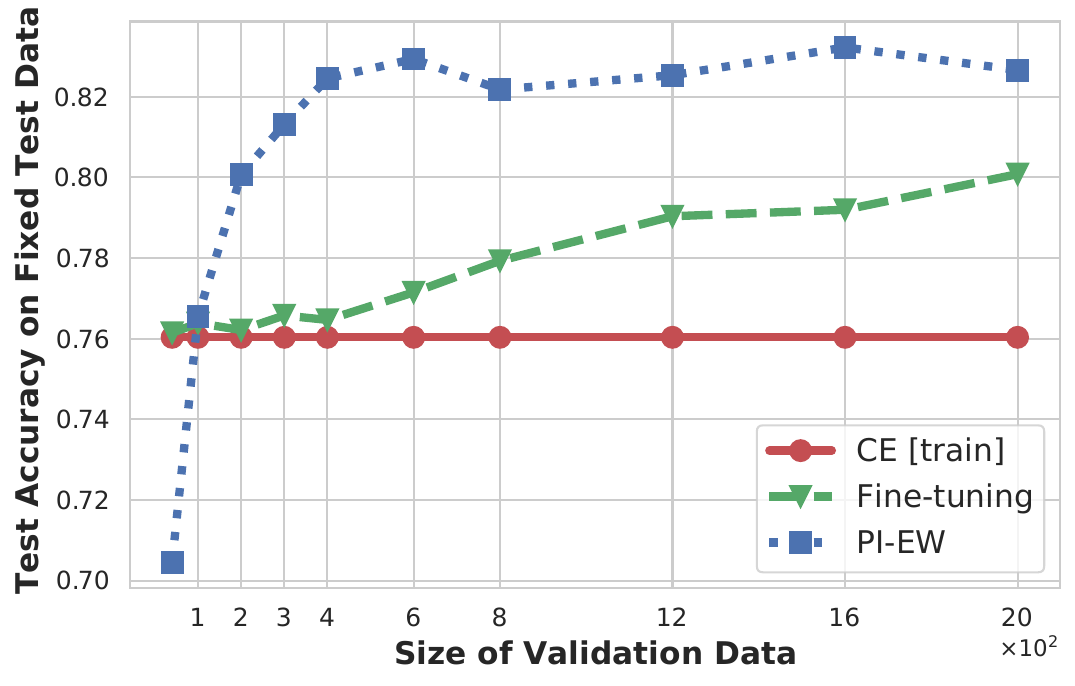}}
		\label{fig:valsizemain}
	}
	\vskip -0.2cm
	\caption{(a) Elicited (class) weights for CIFAR-10 by PI-EW for the default basis (Sec.~\ref{ssec:cifar10}); (b) Run-time for FW-EG and Adaptive Surrogates~\cite{jiang2020optimizing}
    vs no.\ of grouping features on proxy label task (Sec.~\ref{ssec:adult}); (c)  Effect of quality of the base model $\hat\eta^{\tr}$ on Adult (Sec.~\ref{ssec:adult}): as the base model's quality improves, the test accuracies of PI-EW also improves; (d) Effect of the validation set size on Adience (Sec.~\ref{ssec:adience}): PI-EW performs better than fine-tuning even for small validation sets, while both improve with larger ones.
	}
	\label{fig:mainabalation}
\end{figure*}

\begin{table}[t]
    \small
    \centering
    \caption{Black-box fairness metric on the test set for Adult.}
    \vskip -0.25cm
    \begin{tabular}{ll}
    \hline
        Cross-entropy [train] & 0.736 $\pm$	0.005\\
        Cross-entropy [val] & 0.610 $\pm$	0.020 \\
        Learn-to-reweight &  0.729	$\pm$ 0.007 \\
        Fine-tuning & 0.738 $\pm$	0.005\\
        Adaptive Surrogates & 0.812	$\pm$ 0.004\\
        Plug-in [train-val] & 0.812 $\pm$	0.005\\
        \hline
        FW-EG & $\textbf{0.822}	\pm \textbf{0.002}^{***}$\\
        \hline
    \end{tabular}
    \label{tab:adultbb}
    \vskip -0.3cm
\end{table}

We close with two sets of experiments. First, we analyze how the performance of PI-EW, while optimizing accuracy for the Adult experiment (Section~\ref{ssec:adult}), varies with the quality of the base model $\hat{\eta}^\tr$. We save an estimate of $\hat\eta^{\tr}$ after every 50 batches (batch size 32) while training the logistic regression model, and use these estimates as inputs to PI-EW. As shown in Figure~\ref{fig:periodicmain}, the test accuracies for PI-EW improves with the quality of $\hat{\eta}^\tr$ (as measured by the log loss on the hyper-train set).  This is in accordance with Theorem~\ref{thm:iterative-plugin}. One can further improve the quality of the estimate $\eta^{\text{tr}}$ by using calibration techniques~\cite{guo2017calibration}, which will likely enhance the performance of PI-EW as well.

Next, we show that PI-EW is robust to changes in the validation set size when trained on the Adience experiment in Section~\ref{ssec:adience} to optimize accuracy. We set aside 50\% of 6--8 age bucket data for testing, and sample varying sizes of validation data from the rest. As shown in Figure~\ref{fig:valsizemain}, PI-EW generally performs better than fine-tuning even for small validation sets, while both improve with larger ones. The only exception is 100-sized validation set (0.8\% of training data), where we see overfitting due to small validation size.

%% file: discussion.tex
\vspace{-0.2cm}
\section{Conclusion and Discussion}
\label{sec:discussion}
\vskip -0.1cm

We have proposed the FW-EG method for optimizing black-box metrics given query access to the evaluation metric on a small validation set. Our framework includes common distribution shift settings as special cases, and unlike prior distribution correction strategies, is able to handle general non-linear metrics. A key benefit of our method is that it is agnostic to the choice of  $\hat\eta^{\tr}$, and can thus be used to post-shift 
pre-trained deep networks, without having to retrain them. We showed that the post-shift example weights can be flexibly modeled with various choices of basis functions (e.g., hard clusters, RBF kernels, etc.) and  empirically demonstrated their efficacies. We look forward to further improving the results with more nuanced basis functions. 

%% file: ack.tex
\vspace{-0.2cm}
\section*{Acknowledgements}

We thank the anonymous reviewers for their helpful and constructive feedback.
We also thank Google Cloud for supporting this research with cloud computing credits.

%% file: supplement.tex
\renewcommand{\thesection}{\Alph{section}}
\setcounter{section}{0}
\newcommand{\pb}{\vspace*{-\parskip}\noindent\rule[0.5ex]{\linewidth}{1pt}}

\onecolumn

\begin{appendices}

\textbf{Notations:} $\Delta_m$ denotes the $(m-1)$-dimensional simplex. $[m] = \{1,2,\dots,m\}$ represents an index set of size $m$. 
$\|\cdot\|$ denotes the operator norm for matrices and the $\ell_2$ norm for vectors.
For a matrix $\A \in \R^{m\times m}$, $\diag(\A) = [C_{11},\ldots,C_{mm}]^\top$ outputs the $m$ diagonal entries. For an index $j \in [m]$, $\onehot(j) \in \{0,1\}^m$ denotes a one-hot encoding of $j$, and for a classifier $h: \X \> [m]$, $\tilde{h} = \onehot(h)$ denotes the same classifier with one-hot outputs, i.e.\ $\tilde{h}(x) = \onehot(h(x))$.

\section{Extension to General Linear Metrics}
\label{app:linear-gen}
We describe how our proposal extends to black-box metrics $\perf^\Dtrue[h] = \psi(\C[h])$ defined by a function $\psi:[0,1]^{m\times m}\>\R_+$ of \textit{all} confusion matrix entries. 
This handles, for example, the label noise models in Table \ref{tab:correction-weights} with a general (non-diagonal) noise transition matrix $\T$. We begin with metrics that are linear functions of the diagonal and off-diagonal confusion matrix entries $\perf^\Dtrue[h] = \sum_{ij} \beta_{ij}C_{ij}[h]$ for some $\bbeta \in \R^{m\times m}$. 
In this case, we will use an example weighting function $\W: \X \> \R_+^{m\times m}$ that maps an instance $x$ to an $m\times m$ weight matrix $\W(x)$, where $W_{ij}(x) \in \R_+^{m\times m}$ is the weight associated with the $(i,j)$-th confusion matrix entry.

\textbf{\textit{Note that in practice, the metric $\perf^\Dtrue$ may depend on only a subset of $d$ entries of the confusion matrix, in which case, the weighting function only needs to weight those entries. Consequently, the weighting function can be parameterized with  $Ld$ parameters, which can then be estimated by solving a system of $Ld$ linear equations. 
For the sake of completeness, here we describe our approach for  metrics that depend on all $m^2$ confusion entries.}}

\textbf{Modeling weighting function:} Like in \eqref{eq:weighting}, we propose modeling this function as a weighted sum of $L$ basis functions:
\[
W_{ij}(x) \,=\, \sum_{\ell=1}^L \alpha^{\ell}_{ij}\phi^\ell(x),
\]
where each $\phi^\ell:\X\>[0,1]$ and $\alpha^\ell_{ij} \in \R$. Similar to \eqref{eq:example-weights}, our goal is to then estimate coefficients $\balpha$ so that:
\begin{equation}
\E_{(x, y) \sim \Dshift}\Big[\sum_{ij} W_{ij}(x)\,\1(y = i)h_j(x)\Big] \,\approx\, 
\perf^\Dtrue[h],  \forall h.
\label{eq:example-weights-off-diag}
\end{equation}
Expanding the weighting function in \eqref{eq:example-weights-off-diag}, we get:
\begin{equation*}
\sum_{\ell=1}^L\sum_{i,j}
\alpha^{\ell}_{ij}\,
\underbrace{
\E_{(x, y) \sim \Dshift}\big[ \phi^\ell(x)\,\1(y=i)h_j(x) \big]}_{\Phi_{i,j}^{\Dshift, \ell}[h]} \,\approx\, \perf^\Dtrue[h],  \forall h,
\end{equation*}
which can be re-written as:
\begin{equation}
\sum_{\ell=1}^L\sum_{i,j}
\alpha^{\ell}_{ij}
\Phi^{\Dshift, \ell}_{ij}[h] \,\approx\, 
\perf^D[h], \forall h.
\label{append:eq:example-weights-reduced}
\end{equation}

\textbf{Estimating coefficients $\balpha$:} To estimate $\balpha \in \R^{Lm^2}$, our proposal is to probe the metric $\perf^\Dtrue$ at $Lm^2$ different classifiers $h^{\ell,1,1}, \ldots, h^{\ell,m,m}$, with one classifier for each combination $(\ell,i,j)$ of basis functions and confusion matrix entries, and to solve the following  system of $Lm^2$ linear equations:
\begin{align}
\sum_{\ell,i,j}\alpha^{\ell}_{ij}\,
\hat{\Phi}^{\tr, \ell}_{ij}[h^{1,1,1}]&=
\hat{\perf}^\val[h^{1,1,1}]
\nonumber
\\
&\vdots\label{append:eq:system-of-equations-emp}\\
\sum_{\ell,i,j}\alpha^{\ell}_{ij}\,
\hat{\Phi}^{\tr, \ell}_{ij}[h^{L,m,m}]&=
\hat{\perf}^\val[h^{L,m,m}]\nonumber
\end{align}
Here $\hat{\Phi}^{\tr,\ell}_{ij}[h]$ is an estimate of ${\Phi}^{\Dshift,\ell}_{ij}[h]$ using training sample $S^\tr$ and $\hat{\perf}^\val[h]$ is an estimate of $\perf^\Dtrue[h]$ using the validation sample $S^\val$. Equivalently, defining $\hat{\bSigma} \in \R^{Lm^2 \times Lm^2}$ and $\hat{\bcE} \in \R^{Lm^2}$ with each:
\[
\hat{\Sigma}_{(\ell,i,j), (\ell',i',j')} = \hat{\Phi}^{\tr, \ell'}_{i'j'}[h^{\ell,i,j}];~~~
\hat{\perf}_{(\ell,i,j)} = \hat{\perf}^\val[h^{\ell,i,j}],
\]
we compute $\hat{\balpha} = \hat{\bSigma}^{-1}\hat{\bcE}$.

\textbf{Choosing probing classifiers:} As described in Section \ref{subsec:probing-classifier}, we propose picking each probing classifier $h^{\ell,i,j}$ so that the $(\ell,i,j)$-th diagonal entry  of $\hat{\bSigma}$ is large and the off-diagonal entries are all small. This can be framed as the following constrained satisfaction problem:
\begin{center}
For $h^{\ell,i,j}$ pick $h \in \H$ such that:
\begin{align*}
\hat{\Phi}^{\tr,\ell}_{i,j}[h] \geq \gamma,~\text{and}~
\hat{\Phi}^{\tr,\ell'}_{i',j'}[h] \leq \omega, \forall (\ell',i',j') \ne (\ell,i,j),
\end{align*}
\end{center}
for some $0 < \omega < \gamma < 1$. While the more practical  approach prescribed in Section \ref{subsec:probing-classifier} of constructing the probing classifiers from trivial classifiers that predict the same class on all or a subset of examples does not apply here (because here we need to take into account both the diagonal and off-diagonal confusion entries), the above problem {can be solved} using off-the-shelf tools available for rate-constrained optimization problems \cite{cotter2019optimization}.

\textbf{Plug-in classifier:} Having estimated an example weighting function $\hat{\W}: \X \> \R^{m\times m}$, we seek to maximize a weighted objective on the training distribution:
\[
\max_{h}\,\E_{(x, y) \sim \Dshift}\left[\sum_{ij} \hat{W}_{ij}(x)\,\1(y = i)h_j(x)\right],
\]
for which we can construct a plug-in classifier that post-shifts a pre-trained class probability model $\hat{\eta}^\tr: \X \> \Delta_m$:
\[
\widehat{h}(x) \,\in\, \argmax_{j \in [m]} \sum_{i=1}^m\hat{W}_{ij}(x)\,\hat{\eta}^\tr_i(x).
\]

For handling general non-linear metrics $\perf^\Dtrue[h] = \psi(\C[h])$ with a  smooth $\psi:[0,1]^{m\times m}\>\R_+$, we can directly adapt the iterative plug-in procedure in Algorithm \ref{algo:FW}, which would in turn construct a plug-in classifier of the above form in each iteration (line 9). See \citet{narasimhan2015consistent} for more details of the iterative Frank-Wolfe based procedure for optimizing general metrics, where the authors consider non-black-box metrics in the absence of distribution shift.

\section{Proofs}
\subsection{Proof of Theorem \ref{thm:alpha-diagonal-linear-conopt}}
\begin{thm*}[(Restated) \textbf{Error bound on elicited weights}]
Let the input metric be of the form $\hat{\perf}^\lin[h] = \sum_{i}\beta_i \hat{C}^\val_{ii}[h]$ for some (unknown) coefficients $\bbeta \in \R_+^m, \|\bbeta\|\leq 1$.
Let $\perf^\Dtrue[h] = \sum_{i}\beta_i C^\Dtrue_{ii}[h]$. Let $\gamma, \omega > 0$ be such that the constraints in \eqref{eq:con-opt}  are feasible for hypothesis class $\bar{\H}$, for all $\ell, i$. Suppose Algorithm \ref{algo:weight-coeff} chooses each
 classifier $h^{\ell,i}$ to  satisfy \eqref{eq:con-opt}, with $\perf^D[h^{\ell,i}] \in [c, 1], \forall \ell, i$, for some $c>0$.
Let $\bar{\alpha}$ be the associated coefficient in Assumption \ref{asp:alpha-star} for metric $\perf^D$. Suppose $\gamma > 2\sqrt{2}Lm\omega$ and $n^\tr \geq \frac{L^2m\log(Lm|\H|/\delta)}{(\frac{\gamma}{2} - \sqrt{2}Lm\omega)^2}.$
Fix $\delta\in (0,1)$. Then w.p.\  $\geq 1 - \delta$ over draws of $S^\tr$ and $S^\val$ from $\Dshift$ and $\Dtrue$ resp., the coefficients $\hat{\balpha}$ output by Algorithm \ref{algo:weight-coeff} satisfies:
\begin{eqnarray*}
{\|\hat{\balpha} - \bar{\balpha}\| \,\leq\,}
\mathcal{O}\Big(
\frac{Lm}{\gamma^2}\sqrt{\frac{L\log(Lm|\H|/\delta)}{n^\tr}} + 
\frac{\sqrt{Lm}}{\gamma} \Big( \sqrt{\frac{L^2m\log(Lm/\delta)}{c^2\gamma^2 n^\val}} +  \nu\Big)\Big),
\end{eqnarray*}
where the term $|\H|$ can be replaced by a measure of capacity of the hypothesis class $\H$.
\end{thm*}
The solution from Algorithm \ref{algo:weight-coeff} is given by 
 $\hat{\balpha} = \hat{\bSigma}^{-1}\hat{\bcE}$. 
Let $\bar{\balpha}$ be the ``true'' coefficients given in Assumption \ref{asp:alpha-star}. 
Let ${\bSigma} \in \R^{Lm\times Lm}$ denote the population version of  $\hat{\bSigma}$, with
$
\Sigma_{(\ell,i), (\ell',i')} \,=\, \E_{(x,y)\sim \mu}\big[\phi^{\ell'}(x)\1(y=i')h^{\ell,i}_{i'}(x)\big]
$.
Similarly, denote the population version of $\hat{\bcE}$ by:
$
{\perf}_{(\ell, i)} \,=\, \perf^\Dtrue[h^{\ell,i}]
$. 
Let  ${\balpha} = \bSigma^{-1}{\bcE}$ be the solution we obtain had we used the population versions of these quantities.
Further, define 
the vector $\bar{\bcE} \in \R^{Lm}$:
\begin{equation}
\bar{\perf}_{(\ell', i')} = \sum_{\ell,i}
\bar{\alpha}^{\ell}_i
{\Phi}^{\Dshift, \ell}_i[h^{\ell',i'}].
\label{eq:perf-bar}
\end{equation}
It trivially follows that the coefficient $\bar{\balpha}$ given by Assumption \ref{asp:alpha-star} can be written as $\bar{\balpha} = \bSigma^{-1}\bar{\bcE}$. 

We will find the following lemmas useful. Our first two lemmas bound the gap between the empirical and population versions of $\bSigma$ (the left-hand side of the linear system) and $\bcE$ (the right-hand side of the linear system).
\begin{lem}[Confidence bound for $\bSigma$]
\label{lem:Sigma-diff}
Fix $\delta \in (0,1)$. With probability at least $1 - \delta$ over draw of $S^\tr$ from $\Dshift$, 
\begin{eqnarray*}
|\Sigma_{(\ell,i),(\ell',i')} - \hat{\Sigma}_{(\ell,i),(\ell',i')}| \leq 
\mathcal{O}\left(\sqrt{\frac{p_{\ell,i}\log(Lm|\H|/\delta)}{n^\tr}}\right),
\end{eqnarray*}
where $p_{\ell,i} = \E_{(x,y)\sim \mu}[\phi^\ell(x)\1(y=i)]$, 
and consequently,
$$\|\bSigma - \hat{\bSigma}\| \leq \mathcal{O}\left(\sqrt{\frac{L^2m\log(Lm|\H|/\delta)}{n^\tr}}\right).$$
\end{lem}
\begin{proof}
Each row of $\bSigma - \hat{\bSigma}$ contains the difference between the elements $\Phi^{\Dshift,\ell}_{i}[h]$ and $\hat{\Phi}^{\tr,\ell}_{i}[h]$ for a classifier $h$ chosen from $\H$. Using multiplicative Chernoff bounds, we have for a fixed $h$, with probability at least $1 - \delta$ over draw of $S^\tr$ from $\Dshift$
\begin{eqnarray*}
|\Phi^{\Dshift,\ell}_i[h] - \hat{\Phi}^{\tr,\ell}_i[h]| \leq 
\mathcal{O}\left(\sqrt{\frac{p_{\ell,i}\log(1/\delta)}{n^\tr}}\right),
\end{eqnarray*}
where $p_{\ell,i} = \E_{(x,y)\sim \mu}[\phi^\ell(x)\1(y=i)]$.
Taking a union bound over all $h \in \H$, we have with probability at least $1 - \delta$ over draw of $S^\tr$ from $\Dshift$, for any $h \in \H$:
\begin{eqnarray*}
|\Phi^{\Dshift,\ell}_i[h] - \hat{\Phi}^{\tr,\ell}_i[h]| \leq 
\mathcal{O}\left(\sqrt{\frac{p_{\ell,i}\log(|\H|/\delta)}{n^\tr}}\right).
\end{eqnarray*}
Taking a union bound over all $Lm \times Lm$ entries,  we have with probability at least $1 - \delta$, for all $(\ell,i),(\ell',i')$:
\begin{eqnarray*}
|\Sigma_{(\ell,i),(\ell',i')} - \hat{\Sigma}_{(\ell,i),(\ell',i')}| \leq 
\mathcal{O}\left(\sqrt{\frac{p_{\ell,i}\log(Lm|\H|/\delta)}{n^\tr}}\right)
.
\end{eqnarray*}
Upper bounding the operator norm of $\bSigma - \hat{\bSigma}$ with the Frobenius norm, we have
\begin{eqnarray*}
\|\bSigma - \hat{\bSigma}\| &\leq&
\mathcal{O}\left(\sqrt{\frac{\log(Lm|\H|/\delta)}{n^\tr}}\sqrt{\sum_{(\ell,i),(\ell',i')}p_{\ell',i'}}\right)\\&\leq&
\mathcal{O}\left(\sqrt{\frac{\log(Lm|\H|/\delta)}{n^\tr}}\sqrt{\sum_{\ell,i,\ell'}(1)}\right)\,\leq\, \mathcal{O}\left(\sqrt{\frac{L^2m\log(Lm|\H|/\delta)}{n^\tr}}\right),
\end{eqnarray*}
where the second inequality uses the fact that $\sum_{i'}p_{\ell',i'} = \E_{x\sim \P^\mu}\left[\phi^{\ell'}(x)\right]\leq 1$.
\end{proof}

\begin{lem}[Confidence bound for $\bcE$]
Fix $\delta \in (0,1)$. With probability at least $1 - \delta$ over draw of $S^\val$ from $\Dtrue$, 
$$\|\bcE - \hat{\bcE}\| \leq \mathcal{O}\left(\sqrt{\frac{Lm\log(Lm/\delta)}{n^\val}}\right).$$
\label{lem:Perf-diff}
\end{lem}
\begin{proof}
From an application of Hoeffding's inequality, we have for any fixed $h^{\ell,i}$:
\[
|\perf_{(\ell,i)} \,-\, \hat{\perf}_{(\ell,i)}| =
|\perf^\Dtrue[h^{\ell,i}] \,-\, \hat{\perf}^\val[h^{\ell,i}]|
\,=\, \left|\sum_i\beta_i C^\Dtrue_{ii}[h^{\ell,i}] \,-\, \sum_i\beta_i \hat{C}^\val_{ii}[h^{\ell,i}]\right|
\leq \mathcal{O}\left(\sqrt{\frac{\log(1/\delta)}{n^\val}}\right),
\]
which holds with probability at least $1-\delta$ over draw of $S^\val$ and uses the fact that each $\beta_i$ and $C^D_{ii}[h]$ is bounded. 
Taking a union bound over all $Lm$ probing classifiers, we have:
\[
\|\bcE \,-\, \hat{\bcE}\| 
\leq \mathcal{O}\left(\sqrt{Lm}\sqrt{\frac{\log(Lm/\delta)}{n^\val}}\right).
\]
Note that we do not need a uniform convergence argument like in Lemma \ref{lem:Sigma-diff} as the probing classifiers are chosen independent of the validation sample.
\end{proof}

Our last two lemmas show that $\bSigma$ is well-conditioned. We first show that because the probing classifiers $h^{\ell,i}$'s are chosen to satisfy \eqref{eq:con-opt}, the  diagonal and off-diagonal entries of $\bSigma$ can be lower and upper bounded respectively as follows.
\begin{lem}[Bounds on diagonal and off-diagonal entries of $\bSigma$]
Fix $\delta \in (0,1)$. 
With probability at least $1 - \delta$ over draw of $S^\tr$ from $\Dshift$, 
\[
\Sigma_{(\ell, i), (\ell, i)} \geq \gamma \,-\, \mathcal{O}\left(\sqrt{\frac{p_{\ell,i}\log(Lm|\H|/\delta)}{n^\tr}}\right), \forall (\ell,i)
\]
and
\[
\Sigma_{(\ell, i), (\ell', i')} \leq \omega \,+\, \mathcal{O}\left(\sqrt{\frac{p_{\ell,i}\log(Lm|\H|/\delta)}{n^\tr}}\right), \forall (\ell,i) \ne (\ell', i'),
\]
where $p_{\ell,i} = \E_{(x,y)\sim \mu}[\phi^\ell(x)\1(y=i)]$.
\label{lem:Sigma-inv-concentration}
\end{lem}
\begin{proof}
Because the probing classifiers $h^{\ell,i}$'s are chosen from $\H$ to satisfy \eqref{eq:con-opt}, we have $\hat{\Sigma}_{(\ell, i), (\ell, i)} \geq \gamma, \forall (\ell,i)$ and 
$\hat{\Sigma}_{(\ell, i), (\ell', i')} \leq \omega, \forall (\ell,i) \ne (\ell', i').$ The proof follows from generalization bounds similar to Lemma \ref{lem:Sigma-diff}.
\end{proof}
The bounds on the diagonal and off-diagonal entries of $\bSigma$ then allow us to bound its smallest and largest singular values.
\begin{lem}[Bounds on singular values of $\bSigma$]
We have $\|\bSigma\| \,\leq\, L\sqrt{m}$.
Fix $\delta \in (0,1)$. Suppose $\gamma > 2\sqrt{2}Lm\omega$ and $n^\tr \geq \frac{L^2m\log(Lm|\H|/\delta)}{(\frac{\gamma}{2} - \sqrt{2}Lm\omega)^2}.$
With probability at least $1 - \delta$ over draw of $S^\tr$ from $\Dshift$, 
$
\|\bSigma^{-1}\|  \,\leq\, \mathcal{O}\left(\frac{1}{\gamma}\right).
$
\label{lem:Sigma-inv}
\end{lem}
\begin{proof}
We first derive a straight-forward upper bound on the 
the operator norm of $\bSigma$ in terms of its Frobenius norm:
\begin{eqnarray*}
\|\bSigma\| \,\leq\, \sqrt{\sum_{(\ell,i),(\ell',i')}\Sigma^2_{(\ell,i),(\ell',i')}}
\,\leq\, \sqrt{\sum_{(\ell,i),(\ell',i')}p_{\ell',i'}^2}
\,\leq\, \sqrt{\sum_{(\ell,i),(\ell',i')}p_{\ell',i'}}
\,\leq\, \sqrt{\sum_{\ell,i,\ell'}1} \,=\, L\sqrt{m},
\end{eqnarray*}
where $p_{\ell,i} = \E_{(x,y)\sim \mu}[\phi^\ell(x)\1(y=i)]$ and the last inequality uses the fact that $\sum_{i'}p_{\ell',i'} = \E_{x\sim \P^\mu}\left[\phi^{\ell'}(x)\right]\leq 1$.

To bound the operator norm of $\|\bSigma^{-1}\|$, denote
 $\upsilon_{\ell, i} = \mathcal{O}\left(\sqrt{\frac{p_{\ell,i}\log(Lm|\H|/\delta)}{n^\tr}}\right)$.  
From Lemma \ref{lem:Sigma-inv-concentration}, we can express $\bSigma$ as a sum of a matrix $\A$ and a diagonal matrix $\D$, i.e.\ $\bSigma = \A + \D$, where each 
$A_{(\ell, i), (\ell, i)} = 0$,
$A_{(\ell, i), (\ell', i')} \leq \omega + \upsilon_{\ell,i}, \forall (\ell, i) \ne (\ell', i')$ and $D_{(\ell, i), (\ell, i)} \geq \gamma - \upsilon_{\ell,i}$. 
Let $\sigma_{\ell,i}(\bSigma)$ denote the $(\ell,i)$-th largest singular value of $\bSigma$.
By Weyl's inequality, we have that the singular values of $\bSigma$ can be bounded
in terms of the singular values $\D$
(see e.g., \citet{stewart1998perturbation}):
\[
|\sigma_{\ell,i}(\bSigma) - \sigma_{\ell,i}(\D)| \leq \|\A\|,
\]
or
$$
\sigma_{\ell,i}(\D) - \sigma_{\ell,i}(\bSigma) \leq \|\A\|.
$$
We further have:
\allowdisplaybreaks
\begin{eqnarray*}
\sigma_{\ell,i}(\bD)  - \sigma_{\ell,i}(\bSigma) &\leq& \|\A\| \leq \sqrt{\sum_{(\ell,i) \ne (\ell',i')}(\omega+\upsilon_{\ell,i})^2}
+ \upsilon_{\ell, i}
\leq
\sqrt{2}\sqrt{\sum_{(\ell,i) \ne (\ell',i')}\omega^2 + \sum_{(\ell,i) \ne (\ell',i')}\upsilon_{\ell,i}^2}
+ \upsilon_{\ell, i}
\\
&\leq&
\sqrt{2}\sqrt{\sum_{(\ell,i) \ne (\ell',i')}\omega^2} + \sqrt{2}\sqrt{\sum_{(\ell,i) \ne (\ell',i')}\upsilon_{\ell,i}^2}\\
&\leq&
\sqrt{2}Lm\omega \,+\, \mathcal{O}\left(\sqrt{\frac{\log(Lm|\H|/\delta)}{n^\tr}}\right)\sqrt{\sum_{(\ell,i) \ne (\ell',i')}p_{\ell,i}}\\
&\leq& 
\sqrt{2}Lm\omega \,+\, \mathcal{O}\left(\sqrt{\frac{L^2m\log(Lm|\H|/\delta)}{n^\tr}}\right).
\end{eqnarray*}
Since $\sigma_{\ell,i}(\D) \geq \gamma - \max_{\ell, i}\upsilon_{\ell, i}$, and 
$$
\sigma_{\ell,i}(\bSigma) \,\geq\,
\gamma \,-\, \sqrt{2}Lm\omega \,-\, \mathcal{O}\left(\sqrt{\frac{L^2m\log(Lm|\H|/\delta)}{n^\tr}}\right) - \max_{\ell, i}\upsilon_{\ell, i}.
$$
Substituting for $\max_{\ell, i}\upsilon_{\ell, i} \leq \mathcal{O}\left(\sqrt{\frac{\log(Lm|\H|/\delta)}{n^\tr}}\right)$, and denoting $\xi = \sqrt{2}Lm\omega \,+\, \mathcal{O}\left(\sqrt{\frac{L^2m\log(Lm|\H|/\delta)}{n^\tr}}\right)$,
we have $\sigma_{\ell,i}(\bSigma) \,\geq\, \xi$.
With this, we can bound operator norm of $\|\bSigma^{-1}\|$ as:
\[
\|\bSigma^{-1}\| = \frac{1}{\min_{\ell,i} \sigma_{\ell,i}(\bSigma)} \,\leq\,
\frac{1}{\gamma - \xi}\,\leq\,
\mathcal{O}\left(\frac{1}{\gamma}\right),
\]
where the last inequality follows from 
the assumption that $n^\tr \geq \frac{L^2m\log(Lm|\H|/\delta)}{(\frac{\gamma}{2} - \sqrt{2}Lm\omega)^2}$ and hence
$\xi \leq \mathcal{O}\left(\gamma/2\right)$. 
\end{proof}

We are now ready to prove Theorem \ref{thm:alpha-diagonal-linear-conopt}.
\begin{proof}[Proof of Theorem \ref{thm:alpha-diagonal-linear-conopt}]
\allowdisplaybreaks
The solution from Algorithm \ref{algo:weight-coeff} is given by 
 $\hat{\balpha} = \hat{\bSigma}^{-1}\hat{\bcE}$. 
Recall we can write the ``true'' coefficients by
$\bar{\balpha} = \bSigma^{-1}\bar{\bcE}$, 
where $\bar{\bcE}$ is defined in \eqref{eq:perf-bar},
and we also defined $\balpha = \bSigma^{-1}{\bcE}$.
The left-hand side of Theorem \ref{thm:alpha-diagonal-linear-conopt} can then be expanded as:
\begin{eqnarray}
    \|\hat{\balpha} - \bar{\balpha}\| &\leq&
    \|\hat{\balpha} - {\balpha}\| + \|{\balpha} - \bar{\balpha}\|\nonumber\\
    &\leq&
    \|\hat{\balpha} - {\balpha}\| + \|\bSigma^{-1}({\bcE} - \bar{\bcE})\|\nonumber\\
    &\leq&
    \|\hat{\balpha} - {\balpha}\| + \|\bSigma^{-1}\|\|({\bcE} - \bar{\bcE})\|\nonumber\\
    &\leq&
    \|\hat{\balpha} - {\balpha}\| + \nu\sqrt{Lm} \|\bSigma^{-1}\|\label{eq:lhs-penultimate}\\
    &\leq& \|\hat{\balpha} - {\balpha}\| \,+\, \frac{2\nu\sqrt{Lm}}{\gamma}.
    \label{eq:lhs}
\end{eqnarray}

Here the second-last step follows from Assumption \ref{asp:alpha-star}, in particular from
$\left|\sum_{\ell,i}
\bar{\alpha}^{\ell}_i
\Phi^{\Dshift, \ell}_i[h] -
\perf^\Dtrue[h]\right| \,\leq\, \nu, \forall h$, which gives us that
$\left|\sum_{\ell,i}
\bar{\alpha}^{\ell}_i
\Phi^{\Dshift, \ell}_i[h^{\ell',i'}] -
\perf^\Dtrue[h^{\ell',i'}]\right| \,\leq\, \nu$, for all $\ell', i'$. 
The last step follows from Lemma \ref{lem:Sigma-inv} and holds with probability at least $1-\delta$ over draw of $S^\tr$.

All that remains is to bound the term $\|\hat{\balpha} - {\balpha}\|$. 
Given that 
$\hat{\balpha} = \hat{\bSigma}^{-1}\hat{\bcE}$. 
and $\balpha = \bSigma^{-1}{\bcE}$, we can
use standard error analysis for linear systems (see e.g.,
\citet{demmel1997applied}) to bound:

\begin{eqnarray*}
{\|\hat{\balpha} - {\balpha}\|}
&\leq&
\|{\balpha}\|\|\bSigma\|\|\bSigma^{-1}\|\left(
\frac{\|\bSigma - \hat{\bSigma}\|}{\|\bSigma\|}
\,+\,
\frac{\|\bcE - \hat{\bcE}\|}{\|\bcE\|}\right)\\
&\leq&
\|\bSigma^{-1}\|^2\|\bcE\|\left(
\|\bSigma - \hat{\bSigma}\|
\,+\,
\|\bSigma\|\frac{\|\bcE - \hat{\bcE}\|}{\|\bcE\|}\right)
~~(\text{from  $\balpha = \bSigma^{-1}\bcE$})\\
&\leq&
\|\bSigma^{-1}\|^2\|\bcE\|\left(
\|\bSigma - \hat{\bSigma}\|
\,+\,
L\sqrt{m}\frac{\|\bcE - \hat{\bcE}\|}{\|\bcE\|}\right)
~(\text{from Lemma \ref{lem:Sigma-inv}})\\
&\leq&
\|\bSigma^{-1}\|^2\sqrt{Lm}\left(
\|\bSigma - \hat{\bSigma}\|
\,+\,
\frac{L\sqrt{m}}{\sqrt{Lm}c}{\|\bcE - \hat{\bcE}\|}\right)
~(\text{using  $\perf_{(\ell,i)} \in (c,1]$})
\\
&\leq&
\|\bSigma^{-1}\|^2\sqrt{Lm}\left(
\|\bSigma - \hat{\bSigma}\|
\,+\,
\frac{\sqrt{L}}{c}{\|\bcE - \hat{\bcE}\|}\right)\\
&\leq&
\mathcal{O}\left(\frac{\sqrt{Lm}}{\gamma^2}\left(\sqrt{\frac{L^2m\log(Lm|\H|/\delta)}{n^\tr}}
\,+\,
\frac{\sqrt{L}}{c}\sqrt{\frac{Lm\log(Lm/\delta)}{n^\val}}\right)\right)\\
&=&
\mathcal{O}\left(\frac{Lm}{\gamma^2}\left(\sqrt{\frac{L\log(Lm|\H|/\delta)}{n^\tr}}
\,+\,
\frac{1}{c}\sqrt{\frac{L\log(Lm/\delta)}{n^\val}}\right)\right),
\end{eqnarray*}
where the last two steps follow from Lemmas \ref{lem:Sigma-diff}--\ref{lem:Perf-diff} and Lemma \ref{lem:Sigma-inv},
and  hold with probability at least $1-\delta$ over draws of $S^\tr$ and $S^\val$.
Plugging this back into \eqref{eq:lhs} completes the proof.
\end{proof}

\subsection{Error Bound for PI-EW}
\label{app:pi-ew}
When the metric  is linear, we have the following bound on the gap between the metric value achieved by classifier $\hat{h}$ output by Algorithm \ref{algo:linear-metrics}, and the optimal  value. This result will then be useful in proving an error bound for Algorithm \ref{algo:FW} in the next section.
\begin{lem}[\textbf{Error Bound for PI-EW}]
\label{lem:plugin-linear}
Let the input metric be of the form $\hat{\perf}^\lin[h] = \sum_{i}\beta_i \hat{C}^\val_{ii}[h]$ for some (unknown) coefficients $\bbeta \in \R_+^m, \|\bbeta\|\leq 1$, and denote $\perf^\lin[h] = \sum_{i}\beta_i C^\Dtrue_{ii}[h]$.
Let $\bar{\balpha}$ be the associated weighting coefficient for $\perf^\lin$  in Assumption \ref{asp:alpha-star}, with $\|\bar{\balpha}\|_1 \leq B$ and with slack $\nu$. Fix $\delta>0$.
Suppose w.p. $\geq 1-\delta$ over draw of $S^\tr$ and $S^\val$, the weight elicitation routine in line 2 of Algorithm \ref{algo:linear-metrics}
provides coefficients $\hat{\balpha}$ with
 $\|\hat{\balpha} -\bar{\balpha}\| \leq 
 \kappa(\delta, n^\tr, n^\val)
 $, for some function $\kappa(\cdot) > 0$. Let $B' = B + \sqrt{Lm}\,\kappa(\delta, n^\tr, n^\val).$
 Then with the same probability, the classifier $\hat{h}$ output by Algorithm \ref{algo:linear-metrics} satisfies:
\begin{eqnarray*}
{
\max_{h}\perf^\lin[h] - \perf^\lin[\hat{h}]} &\leq 
B'\E_x\left[\|\eta^\tr(x)-
\hat{\eta}^\tr(x)\|_1\right]
\,+\,
 2\sqrt{Lm}\,\kappa(\delta, n^\tr, n^\val)
  \,+\, 2\nu,
\end{eqnarray*}
where $\eta_i^\tr(x) = \P^\mu(y=i|x)$. Furthermore, when the metric coefficients $\|\bbeta\| \leq Q$, for some $Q > 0$, then 
\begin{eqnarray*}
{
\max_{h}\perf^\lin[h] - \perf^\lin[\hat{h}]} &\leq 
Q\left(B'\E_x\left[\|\eta^\tr(x)-
\hat{\eta}^\tr(x)\|_1\right]
\,+\,
 2\sqrt{Lm}\,\kappa(\delta, n^\tr, n^\val)\,+\, 2\nu\right).
\end{eqnarray*}
\end{lem}

\begin{proof}
For the proof, we will treat $\widehat{h}$ as a classifier that outputs one-hot labels,
i.e.\ as classifier  $\widehat{h}: \X \> \{0,1\}^m$ with
\begin{equation}
    \widehat{h}(x) \,=\, \onehot\left(\argmax^*_{i \in [m]} \widehat{W}_{i}(x)\hat{\eta}^\tr_i(x)\right),
    \label{eq:h-hat-onehot}
\end{equation}
where $\argmax^*$ breaks ties in favor of the largest class.

Let
$\bar{W}_i(x)  = \sum_{\ell=1}^L \bar{\alpha}^{\ell}_{i}\phi^\ell(x)$ and $\hat{W}_i(x) = \sum_{\ell=1}^L \hat{\alpha}^{\ell}_{i}\phi^\ell(x)$. It is easy to see that 
\begin{equation}
\label{eq:w-bar-w-hat}
|\bar{W}_i(x) - \hat{W}_i(x)| \leq 
\|\bar{\balpha} - \hat{\balpha}\|\sqrt{\sum_{\ell=1}^L \phi^\ell(x)^2} \leq
\sqrt{Lm}\|\bar{\balpha} - \hat{\balpha}\| \leq \sqrt{Lm}\kappa,
\end{equation}
where in the second inequality we use $|\phi^\ell(x)| \leq 1$, and in the last inequality, we have shortened the notation $\kappa(\delta, n^\tr, n^\val)$ to $\kappa$ and for simplicity will avoid mentioning that this holds with high probability.

\allowdisplaybreaks
Further, recall from Assumption \ref{asp:alpha-star} that
\begin{equation*}
|\bar{W}_i(x)| \leq \|\bar{\balpha}\|_1\max_{\ell}|\phi^\ell(x)| \leq B(1) = B
\end{equation*}
and so from \eqref{eq:w-bar-w-hat},
\begin{equation}
|\hat{W}_i(x)| \leq B + \sqrt{Lm}\kappa.
\label{eq:W-bound}
\end{equation}
We also have from Assumption \ref{asp:alpha-star} that
\begin{equation*}
\left|\perf^\lin[h] \,-\,  \E_{(x,y)\sim \mu}\left[\sum_{i=1}^m\bar{W}_i(x)\1(y=i)h_i(x)\right]\right| \leq \nu, \forall h.
\end{equation*}
Equivalently, this can be re-written in terms of the conditional class probabilities $\eta^\tr(x) = \P^\mu(y=1|x)$:
\begin{equation}
   \left|\perf^\lin[h] \,-\,  \E_{x\sim \P^\mu}\left[\sum_{i=1}^m\bar{W}_i(x)\eta^\tr_i(x)h_i(x)\right]\right| \leq \nu, \forall h,
\label{eq:perf-ex-over-P-mu} 
\end{equation}
where $\P^\mu$ denotes the marginal distribution of $\Dshift$ over $\X$.
Denoting $h^* \in \argmax_{h}\,\perf^\lin[h]$, we then have from \eqref{eq:perf-ex-over-P-mu},
\begin{eqnarray*}
\allowdisplaybreaks
\lefteqn{\max_{h}\,\perf^\lin[h] \,-\, \perf^\lin[\hat{h}]}\\
&=&
\sum_{i=1}^m \E_x\left[\bar{W}_{i}(x)\eta^\tr_{i}(x)h^*_i(x))\right]
\,-\, \sum_{i=1}^m \E_x\left[ \bar{W}_{i}(x)\eta^\tr_{i}(x)\hat{h}_i(x)\right] \,+\, 2\nu
\\
&\leq&
\sum_{i=1}^m \E_x\left[\hat{W}_{i}(x)\eta^\tr_{i}(x)h^*_i(x))\right]
\,-\, \sum_{i=1}^m \E_x\left[ \hat{W}_{i}(x)\eta^\tr_{i}(x)\hat{h}_i(x)\right] \,+\, 2\nu \,+\,
2\sqrt{Lm}\kappa\\
&&
\hspace{6cm}\text{(from \eqref{eq:w-bar-w-hat}, $\textstyle\sum_{i=1}^m\eta^\tr_i(x) = 1$ and $h_i(x) \leq 1$)}\\
&\leq&
\sum_{i=1}^m \E_x\left[\hat{W}_{i}(x)\eta^\tr_{i}(x)h^*_i(x))\right]
\,-\, \sum_{i=1}^m\E_x\left[\hat{W}_{i}(x)\hat{\eta}^\tr_{i}(x)h^*_i(x))\right]\\
&&
\hspace{2cm}
\,+\, \sum_{i=1}^m\E_x\left[\hat{W}_{i}(x)\hat{\eta}^\tr_{i}(x)h^*_i(x))\right]
\,-\, \sum_{i=1}^m \E_x\left[ \hat{W}_{i}(x)\eta^\tr_{i}(x)\hat{h}_i(x)\right] \,+\, 2\nu \,+\,
2\sqrt{Lm}\kappa
\end{eqnarray*}
From the definition of $\widehat{h}$ in \eqref{eq:h-hat-onehot}, we 
have that $\sum_{i=1}^m\hat{W}_{i}(x)\hat{\eta}^\tr_{i}(x)\hat{h}_i(x) \geq \sum_{i=1}^m\hat{W}_{i}(x)\hat{\eta}^\tr_{i}(x)h_i(x),$ for all $h: \X \> \Delta_m$. Therefore,
\begin{eqnarray*}
\lefteqn{\max_{h}\,\perf^\lin[h] \,-\, \perf^\lin[\hat{h}]}\\
&\leq&
\sum_{i=1}^m \E_x\left[\hat{W}_{i}(x)\eta^\tr_{i}(x)h^*_i(x))\right]
\,-\, \sum_{i=1}^m\E_x\left[\hat{W}_{i}(x)\hat{\eta}^\tr_{i}(x)h^*_i(x))\right]\\
&&
\hspace{2cm}
\,+\, \sum_{i=1}^m\E_x\left[\hat{W}_{i}(x)\hat{\eta}^\tr_{i}(x)\hat{h}_i(x))\right]
\,-\, \sum_{i=1}^m \E_x\left[ \hat{W}_{i}(x)\eta^\tr_{i}(x)\hat{h}_i(x)\right] \,+\, 2\nu \,+\,
2\sqrt{Lm}\kappa
\\
&\leq&
\sum_{i=1}^m \E_x\left[\hat{W}_{i}(x)|\eta^\tr_{i}(x) - \hat{\eta}^\tr_i(x)|h^*_i(x)\right]
\,+\, \sum_{i=1}^m\E_x\left[\hat{W}_{i}(x)|\eta^\tr_{i}(x) - \hat{\eta}^\tr_i(x)|\hat{h}_i(x)\right]\,+\, 2\nu \,+\,
2\sqrt{Lm}\kappa
\\
&\leq&
\sum_{i=1}^m \E_x\left[\hat{W}_{i}(x)|\eta^\tr_{i}(x) - \hat{\eta}^\tr_i(x)||h^*_i(x) - \hat{h}_i(x)|\right]\,+\, 2\nu \,+\,
2\sqrt{Lm}\kappa
\\
&\leq&
\E_x\left[\max_i \left(\hat{W}_{i}(x)|h^*_i(x) - \hat{h}_i(x)|\right)\|\eta(x)-
\hat{\eta}(x)\|_1\right]\,+\, 2\nu \,+\,2\sqrt{Lm}\kappa
\\
&\leq&
(B + \sqrt{Lm}\kappa)\,\E_x\left[\|\eta(x)-
\hat{\eta}(x)\|_1\right]\,+\, 2\nu \,+\, 2\sqrt{Lm}\kappa,
\end{eqnarray*}
where the last step follows from \eqref{eq:W-bound} and $|h_i(x) - \hat{h}_i(x)|\leq 1$. This completes the proof. The second part, where $\|\bbeta\| \leq Q$, follows by applying Assumption \ref{asp:alpha-star} to normalized coefficients $\bbeta/\|\bbeta\|$, and scaling the associated slack $\nu$ by $Q$.
\end{proof}

\subsection{Proof of Theorem \ref{thm:iterative-plugin}}
\label{app:proof-FW}
We will make a couple of minor changes to the algorithm to simplify the analysis. Firstly, instead of using the same sample $S^\val$ for both estimating the example weights (through call to \textbf{PI-EW} in line 9) and estimating confusion matrices $\hat{\C}^{\val}$ (in line 10), we split $S^\val$ into two halves, use one half for the first step and the other half for the second step. Using independent samples for the two steps, we will be able to derive straight-forward confidence bounds on the estimated confusion matrices in each case. In our experiments however, we find the algorithm to be effective even when a common sample is used for both steps. Secondly, we modify line 8 to include a shifted version of the metric $\hat{\perf}^\val$, so that later in Appendix \ref{app:complex-unknown} when we handle the case of ``unknown $\psi$'', we can avoid having to keep track of an additive constant in the gradient coefficients.
\renewcommand{\thealgorithm}{3$^*$}
\begin{algorithm}[H]
\caption{\hspace{-0.075cm}\textbf{:} \textbf{F}rank-\textbf{W}olfe with \textbf{E}licited \textbf{G}radients (\textbf{FW-EG}) for General Diagonal Metrics} \label{algo:FW-modified}
\begin{algorithmic}[1]
\STATE \textbf{Input:} $\hat{\perf}^\val$, Basis functions $\phi^1, \ldots, \phi^L: \X \> [0,1]$, Pre-trained $\hat{\eta}^\tr: \X \> \Delta_m$,
$S^\tr \sim \Dshift$, 
$S^\val \sim \Dtrue$ split into two halves
$S^\val_1$ and $S^\val_2$ of sizes $\lceil n^\val/2\rceil$ and $\lfloor n^\val/2\rfloor$ respectively, $T$, $\epsilon$
\STATE Initialize classifier $h^0$ and $\c^0 = \diag(\widehat{\C}^\val[h^0])$
\STATE \textbf{For} $t =  0$ \textbf{to} $T-1$ \textbf{do}
\STATE ~~~\textbf{if} $\perf^\Dtrue[h] = \psi(C_{11}^D[h],\ldots,C_{mm}^D[h])$ for known $\psi$:
\STATE ~~~~~~~$\bbeta^{t}\,=\, \nabla\psi(\c^{t})$
\STATE ~~~~~~~$\hat{\perf}^\lin[h]\,=\, \sum_i \beta^t_i \hat{C}^\val_{ii}[h],$  evaluated using $S^\val_1$
\STATE ~~~\textbf{else}
\STATE ~~~~~~~$\hat{\perf}^\lin[h]= \hat{\perf}^\val[h]- \hat{\perf}^\val[h^t]$, evaluated using $S^\val_1$ \hspace{0.1cm}\COMMENT{small $\epsilon$ recommended}
\STATE ~~~$\widehat{f} = \text{\textbf{PI-EW}}(\hat{\perf}^\lin, \phi^1,..., \phi^L, \hat{\eta}^\tr, S^\tr, S^\val_1, h^{t}, \epsilon)$
\STATE ~~~$\tilde{\c} = \diag(\widehat{\C}^\val[\widehat{f}])$, evaluated using $S^\val_2$
\STATE ~~~${h}^{t+1} = \big(1-\frac{2}{t+1}\big) {h}^{t} + \frac{2}{t+1} \onehot(\widehat{f})$
\STATE ~~~${\c}^{t+1} = \big(1-\frac{2}{t+1}\big) {\c}^{t} + \frac{2}{t+1}\tilde{\c}$
\STATE \textbf{End For}
\STATE \textbf{Output:} $\hat{h} = h^T$
\end{algorithmic}
\end{algorithm}

\begin{thm*}[(Restated) \textbf{Error Bound for FW-EG with known $\psi$}]
Let $\perf^\Dtrue[h] = \psi(C^\Dtrue_{11}[h],\ldots, C^\Dtrue_{mm}[h])$ for a \emph{known} concave function $\psi: [0,1]^m \>\R_+$, which is $Q$-Lipschitz, and  $\lambda$-smooth w.r.t.\ the $\ell_1$-norm.  
Let $\hat{\perf}^\val[h]=\psi(\hat{C}^\val_{11}[h],\ldots, \hat{C}^\val_{mm}[h])$. 
Fix $\delta \in (0, 1)$.
Suppose Assumption \ref{asp:alpha-star} holds with slack $\nu$, and for any linear metric $\sum_i\beta_i C^\Dtrue_{ii}[h]$ with $\|\bbeta\|\leq 1$, whose associated weight coefficients is $\bar{\balpha}$ with $\|\bar{\balpha}\| \leq B$, 
 w.p. $\geq 1-\delta$ over draw of $S^\tr$ and $S^\val_1$, the weight elicitation routine in Algorithm \ref{algo:weight-coeff} outputs coefficients $\hat{\balpha}$ with
 $\|\hat{\balpha} -\bar{\balpha}\| \leq 
 \kappa(\delta, n^\tr, n^\val)
 $, for some function $\kappa(\cdot) > 0$. Let $B' = B + \sqrt{Lm}\,\kappa(\delta/T, n^\tr, n^\val).$ Assume $m \leq n^\val$.
 Then w.p.\  $\geq 1 - \delta$ over draws of $S^\tr$ and $S^\val$ from $\Dtrue$ and $\Dshift$ resp., the classifier $\hat{h}$ output by Algorithm \ref{algo:FW-modified} after $T$ iterations satisfies:
\begin{eqnarray*}
\lefteqn{
\max_{h}\perf^D[h] - \perf^D[\hat{h}]\,\leq\,
2QB'\E_x\left[\|\eta^\tr(x)-
\hat{\eta}^\tr(x)\|_1\right] +
4Q\nu +
 4Q\sqrt{Lm}\,\kappa(\delta/T, n^\tr, n^\val)}\\
 &&
 \hspace{4.5cm}
+\, \mathcal{O}\left(\lambda m\sqrt{\frac{m\log(n^\val)\log(m) + \log(m/\delta)}{n^\val}} + \frac{\lambda}{T}\right).
\end{eqnarray*}
\end{thm*}
The proof adapts techniques from \citet{narasimhan2015consistent}, who
show guarantees for a Frank-Wolfe based learning algorithm with a known $\psi$ in the \textit{absence} of distribution shift.
The main proof steps are listed below:
\begin{itemize}
    \item Prove a generalization bound for the confusion matrices $\hat{\C}^\val$ evaluated in line 10 on the validation sample (Lemma \ref{lem:C-genbound})
    \item Establish an error bound for the call to \textbf{PI-EW} in line 9 (Lemma \ref{lem:plugin-linear} in previous section)
    \item Combine the above two results to show that the classifier $\hat{f}$ returned in line 9 is an approximate linear maximizer needed by the Frank-Wolfe algorithm (Lemma \ref{lem:lmo})
    \item Combine Lemma \ref{lem:lmo} with a convergence guarantee for the outer Frank-Wolfe algorithm  \cite{narasimhan2015consistent, Jaggi13} (using convexity of the space of confusion matrices $\cC$) to complete the proof (Lemmas \ref{lem:C-convexity}--\ref{lem:fw}).
\end{itemize}

\begin{lem}[Generalization bound for $\C^\Dtrue$]
\label{lem:C-genbound}
Fix $\delta \in (0,1)$. 
Let $\hat{\eta}^\tr: \X \> \Delta_m$ be a fixed class probability estimator. Let $\mathcal{G} = \{h:\X\>[m]\,|\,h(x) \in \argmax_{i\in[m]}\beta_i\hat{\eta}^\tr_i(x) \text{ for some }\bbeta \in \R_+^m\}$ be the set of plug-in classifiers defined with $\hat{\eta}^\tr$. 
Let $$\bar{\mathcal{G}} = \{h(x) = \textstyle\sum_{t=1}^T u_t h_t(x)\,|\, T \in \N, h_1,\ldots,h_T \in \mathcal{G}, \mathbf{u} \in \Delta_T\}$$ be the set of all randomized classifiers constructed from a finite number of plug-in classifiers in $\mathcal{G}$. 
Assume $m \leq n^\val$. 
Then with probability at least $1 - \delta$ over draw of $S^\val$ from $\Dtrue$, then for $h\in \bar{\mathcal{G}}$:
$$\|\C^\Dtrue[h] - \hat{\C}^\val[h]\|_\infty \leq \mathcal{O}\left(\sqrt{\frac{m\log(m)\log(n^\val) + \log(m/\delta)}{n^\val}}\right).$$
\end{lem}
\begin{proof}
The proof follows from standard convergence based generalization arguments, where we bound the capacity of the class of plug-in classifiers $\mathcal{G}$ in terms of its Natarajan dimension \cite{natarajan1989learning, daniely2011multiclass}.
Applying Theorem 21 from \cite{daniely2011multiclass}, we have that the Natarajan dimension  of $ \mathcal{G}$ is at most $d = m\log(m)$. 
Applying the generalization bound in Theorem 13 in \citet{daniely2015multiclass}, along with the assumption that $m\leq n^\val$, we have for any $i \in [m]$,
with probability at least $1 - \delta$ over draw of $S^\val$ from $\Dtrue$,  for any $h\in \mathcal{G}$:
$$|C_{ii}^\Dtrue[h] - \hat{C}_{ii}^\val[h]| \leq \mathcal{O}\left(\sqrt{\frac{m\log(m)\log(n^\val) + \log(1/\delta)}{n^\val}}\right).$$
Further note that for any randomized classifier $\bar{h}(x) = \sum_{t=1}^T u_t h_t(x) \in\bar{\mathcal{G}},$ for some $\mathbf{u} \in \Delta_T$,
$$|C_{ii}^\Dtrue[\bar{h}] - \hat{C}_{ii}^\val[\bar{h}]| \leq
\sum_{t=1}^Tu_t|C_{ii}^\Dtrue[h_t] - \hat{C}_{ii}^\val[h_t]|
\leq
\mathcal{O}\left(\sqrt{\frac{m\log(m)\log(n^\val) + \log(1/\delta)}{n^\val}}\right),$$
where the first inequality follows from linearity of expectations. 
Taking a union bound over all diagonal entries $i \in[m]$ completes the proof. 
\end{proof}

We next show that the call to \textbf{PI-EW} in line 9 of Algorithm \ref{algo:FW} computes an approximate  maximizer $\hat{f}$ for $\hat{\perf}^\lin$. This is an extension of Lemma 26 in \citet{narasimhan2015consistent}.
\begin{lem}[Approximation error in linear maximizer $\hat{f}$]
\label{lem:lmo} 
For each iteration $t$ in Algorithm \ref{algo:FW}, denote $\bar{\c}^t = \diag(\C^D[h^t])$, 
and $\bar{\bbeta}^t = \nabla\psi(\bar{\c}^t)$. Suppose the assumptions in Theorem \ref{thm:iterative-plugin} hold. Let $B' = B + \sqrt{Lm}\,\kappa(\delta, n^\tr, n^\val)$. 
Assume $m \leq n^\val$. Then w.p.\ $\geq 1 - \delta$ over draw of $S^\tr$ and $S^\val$ from $\Dshift$ and $\Dtrue$ resp., for any $t = 1, \ldots, T$, the classifier $\hat{f}$ returned by \emph{\textbf{PI-EW}} in line 9 satisfies:
\begin{eqnarray*}
\lefteqn{
\max_{h}\sum_{i}\bar{\beta}^t_i C^\Dtrue_{ii}[h] \,-\, 
\sum_{i}\bar{\beta}^t_iC^\Dtrue_{ii}[\hat{f}] \,\leq\,
QB'\E_x\left[\|\eta^\tr(x)-
\hat{\eta}^\tr(x)\|_1\right] 
\,+\, 2Q\nu
}\\
 &&
 \hspace{1cm}
 +\,
 2Q\sqrt{Lm}\,\kappa\left(\textstyle\frac{\delta}{T}, n^\tr, n^\val\right)
 \,+\, 
 \mathcal{O}\left(\lambda m\sqrt{\frac{m\log(m)\log(n^\val) + \log(m/\delta)}{n^\val}}\right).
\end{eqnarray*}
\end{lem}
\begin{proof}
The proof uses Theorem \ref{thm:alpha-diagonal-linear-conopt} to bound the approximation errors in the linear maximizer $\hat{f}$ (coupled with a union bound over $T$ iterations), and  
Lemma \ref{lem:C-genbound} to bound the estimation errors in the confusion matrix $\c^t$ used to compute the gradient $\nabla\psi(\c^t)$. 

Recall from Algorithm \ref{algo:FW} that $\c^t = \diag(\hat{\C}^\val[h^t])$ and $\bbeta^t =\nabla \psi(\c^t)$.
Note that these are approximations to the actual quantities we are interested in $\bar{\c}^t = \diag(\C^D[h^t])$
and $\bar{\bbeta}^t = \nabla\psi(\bar{\c}^t)$, both of which are evaluated using the population confusion matrix.
Also, $\|\bbeta\| = \|\nabla\psi(\c^t)\| \leq Q$ from $Q$-Lipschitzness of $\psi$.

Fix iteration $t$, and let $h^* \in \argmax_{h}\sum_{i}\bar{\beta}^t_i C^\Dtrue_{ii}[h]$ for this particular iteration. Then:
\allowdisplaybreaks
\begin{align*}
\lefteqn{\sum_{i}\bar{\beta}^t_i C^\Dtrue_{ii}[h^*] \,-\, 
\sum_{i}\bar{\beta}^t_iC^\Dtrue_{ii}[\hat{f}]}\\
&= 
\sum_{i}\bar{\beta}^t_i C^\Dtrue_{ii}[h^*] \,-\, 
\sum_{i}{\beta}^t_iC^\Dtrue_{ii}[h^*]
\,+\,
\sum_{i}{\beta}^t_iC^\Dtrue_{ii}[h^*] \,-\,
\sum_{i}{\beta}^t_i C^\Dtrue_{ii}[\hat{f}] \,+\, 
\sum_{i}{\beta}^t_iC^\Dtrue_{ii}[\hat{f}] \,-\,
\sum_{i}\bar{\beta}^t_iC^\Dtrue_{ii}[\hat{f}]\\
&\leq
\|\bbeta^t - \bar{\bbeta}^t\|_\infty\sum_{i}C_{ii}^D[h^*]\,+\,
 \sum_{i}{\beta}^t_iC^\Dtrue_{ii}[h^*] \,-\,
\sum_{i}{\beta}^t_i C^\Dtrue_{ii}[\hat{f}] \,+\, 
\|\bbeta^t - \bar{\bbeta}^t\|_\infty\sum_{i}C_{ii}^D[\hat{f}]\\
&\leq
\|\bbeta^t - \bar{\bbeta}^t\|_\infty(1)\,+\,
\max_h \sum_{i}{\beta}^t_iC^\Dtrue_{ii}[h] \,-\,
\sum_{i}{\beta}^t_i C^\Dtrue_{ii}[\hat{f}] \,+\, 
\|\bbeta^t - \bar{\bbeta}^t\|_\infty(1)
~~~~\text{(because $\textstyle\sum_{i,j}C^D_{ij}[h] = 1$)}
\\
&=
2\|\bbeta^t - \bar{\bbeta}^t\|_\infty\,+\,
\max_h \sum_{i}{\beta}^t_iC^\Dtrue_{ii}[h] \,-\,
\sum_{i}{\beta}^t_i C^\Dtrue_{ii}[\hat{f}]\\
&=
2\|\nabla\psi(\c^t)- \nabla\psi(\bar{\c}^t)\|_\infty\,+\,
\max_h \sum_{i}{\beta}^t_iC^\Dtrue_{ii}[h] \,-\,
\sum_{i}{\beta}^t_i C^\Dtrue_{ii}[\hat{f}]\\
&\leq
2\lambda\|\c^t- \bar{\c}^t\|_1\,+\,
\max_h \sum_{i}{\beta}^t_iC^\Dtrue_{ii}[h] \,-\,
\sum_{i}{\beta}^t_i C^\Dtrue_{ii}[\hat{f}]
~~~~\text{(because $\psi$ is $\lambda$-smooth w.r.t.\ the $\ell_1$ norm)}\\
&\leq
2\lambda m\|\c^t- \bar{\c}^t\|_\infty\,+\,
\max_h \sum_{i}{\beta}^t_iC^\Dtrue_{ii}[h] \,-\,
\sum_{i}{\beta}^t_i C^\Dtrue_{ii}[\hat{f}]\\
&\leq
\mathcal{O}\left(\lambda m\sqrt{\frac{m\log(m)\log(n^\val) + \log(m/\delta)}{n^\val}}\right)+
QB'\E_x\left[\|\eta^\tr(x)-
\hat{\eta}^\tr(x)\|_1\right] +
 \\ & \hspace{1cm} 2Q\sqrt{Lm}\,\kappa(\delta, n^\tr, n^\val) + 2Q\nu,
 \numberthis\label{eq:approx-lin-max-last}
\end{align*}
where $B' = B + \sqrt{Lm}\,\kappa(\delta, n^\tr, n^\val)$. The last step 
holds with probability at least $1-\delta$ over draw of $S^\val$ and $S^\tr$, and
follows from 
Lemma \ref{lem:C-genbound} and
Lemma \ref{lem:plugin-linear} (using  $\|\bbeta^t\| \leq Q$). 
The first bound on $\|\c^t- \bar{\c}^t\|_\infty = \|\hat{\C}^\val[h^t] - {\C}^\Dtrue[h^t]\|_\infty$ holds for any randomized classifier $h^t$ constructed from a finite number of plug-in classifiers. The second bound on the linear maximization errors holds only for a fixed $t$, and so
we need to take a union bound over all iterations $t = 1, \ldots, T$, to complete the proof. Note that
because we use two independent samples $S^\val_1$ and $S^\val_2$ for the two bounds, they each hold with high probability over draws of $S^\val_1$ and $S^\val_2$ respectively, and hence with high probability over draw of $S^\val$.
\end{proof}

Our last two lemmas restate results from \citet{narasimhan2015consistent}. The first shows  convexity of the space of confusion matrices (Proposition 10 from their paper), and the second applies  a result from \citet{Jaggi13} to show convergence of the classical Frank-Wolfe algorithm with approximate linear maximization steps (Theorem 16 in \citet{narasimhan2015consistent}).
\begin{lem}[{Convexity of space of confusion matrices}]
\label{lem:C-convexity}
Let $\cC = \{\diag(\C^D[h])\,|\,h: \X \> \Delta_m\}$ denote the set of all confusion matrices achieved by some randomized classifier $h: \X \> \Delta_m$. Then $\cC$ is convex.
\end{lem}
\begin{proof}
For any $\C^1, \C^2 \in \cC$, $\exists h_1, h_2: \X \> \Delta_m$ such that $\c^1 = \diag(\C^D[h_1])$ and $\c^2 = \diag(\C^D[h_2])$. We need to show that for any $u \in [0,1]$, $u\c^1 + (1-u)\c^2 \in \cC$. This is true because the randomized classifier $h(x) = u h_1(x) + (1-u)h_2(x)$ yields a confusion matrix $\diag(\C^D[h]) = u\,\diag(\C^D[h_1]) + (1-u)\diag(\C^D[h_2]) = u\c^1 + (1-u)\c^2 \in \cC$.
\end{proof}
\begin{lem}[{Frank-Wolfe with approximate linear maximization} \cite{narasimhan2015consistent}]
\label{lem:fw}
Let $\perf^\Dtrue[h] = \psi(C^\Dtrue_{11}[h],\ldots, C^\Dtrue_{mm}[h])$ for a  concave function $\psi: [0,1]^m \>\R_+$ that is $\lambda$-smooth w.r.t.\ the $\ell_1$-norm. 
For each iteration $t$, define $\bar{\bbeta}^t = \nabla\psi(\diag(\C^D[h^t]))$.
Suppose line 9 of Algorithm \ref{algo:FW} returns a classifier $\hat{f}$ such that $\max_{h}\sum_{i}\bar{\beta}^t_i C^\Dtrue_{ii}[h] \,-\,
\sum_{i}\bar{\beta}^t_iC^\Dtrue_{ii}[\hat{f}] \leq \Delta,\forall t \in [T]$. Then the  classifier $\hat{h}$ output by Algorithm \ref{algo:FW} after $T$ iterations satisfies:
\begin{eqnarray*}
\max_{h}\perf^D[h] - \perf^D[\hat{h}]\,\leq\,
2\Delta
+\, \frac{8\lambda}{T+2}.
\end{eqnarray*}
\end{lem}

\begin{proof}[Proof of Theorem \ref{thm:iterative-plugin}]

The proof follows by plugging in the result from Lemma \ref{lem:lmo} into Lemma \ref{lem:fw}.
\end{proof}
\section{Error Bound for Weight Elicitation with Fixed Probing Classifiers}
\label{app:norm-sigma-bound}
We first state a general error bound for Algorithm \ref{algo:weight-coeff} in terms of the singular values of $\bSigma$ for any \textit{fixed} choices for the probing classifiers. We then bound the singular values for the fixed choices in \eqref{eq:trivial-classifiers} under some specific assumptions.
\begin{thm}[\textbf{Error bound on elicited weights with fixed probing classifiers}]
Let $\perf^\Dtrue[h] = \sum_{i}\beta_i C^\Dtrue_{ii}[h]$ for some (unknown) $\bbeta \in \R^m$, and let 
$\hat{\perf}^\val[h] = \sum_{i}\beta_i \hat{C}^\val_{ii}[h]$. 
Let $\bar{\balpha}$ be the associated coefficient in Assumption \ref{asp:alpha-star} for metric $\perf^D$. Fix $\delta\in (0,1)$. Then for any fixed choices of the probing classifiers $h^{\ell,i}$, we have with probability  $\geq 1 - \delta$ over draws of $S^\tr$ and $S^\val$ from $\Dshift$ and $\Dtrue$ resp., the coefficients $\hat{\balpha}$ output by Algorithm \ref{algo:weight-coeff} satisfies:
\begin{eqnarray*}
{\|\hat{\balpha} - \bar{\balpha}\| \,\leq\,}
\mathcal{O}\left(\frac{1}{\sigma_{\min}(\bSigma)^2}\left(Lm\sqrt{\frac{L\log(Lm/\delta)}{n^\tr}}
\,+\,
\sigma_{\max}(\bSigma)\sqrt{\frac{Lm\log(Lm/\delta)}{n^\val}}\right) \,+\, \frac{\nu\sqrt{Lm}}{\sigma_{\min}(\bSigma)}\right),
\end{eqnarray*}
where $\sigma_{\min}(\bSigma)$ and $\sigma_{\min}(\bSigma)$
are respectively the smallest and largest singular values of $\bSigma$.
\end{thm}
\begin{proof}
The proof follows the same steps as Theorem \ref{thm:alpha-diagonal-linear-conopt}, except for the bound on  $\|\hat{\balpha} - {\balpha}\|$. Specifically, we have from \eqref{eq:lhs-penultimate}:
\begin{equation}
{\|\hat{\balpha} - \bar{\balpha}\|}
\,\leq\, {\|\hat{\balpha} - {\balpha}\|} \,+\, \nu\sqrt{Lm}\|\bSigma^{-1}\|.
\label{eq:lhs-fixed}
\end{equation}
We next bound: 
\begin{eqnarray*}
{\|\hat{\balpha} - {\balpha}\|}
&\leq&
\|{\balpha}\|\|\bSigma\|\|\bSigma^{-1}\|\left(
\frac{\|\bSigma - \hat{\bSigma}\|}{\|\bSigma\|}
\,+\,
\frac{\|\bcE - \hat{\bcE}\|}{\|\bcE\|}\right)\\
&\leq&
\|\bSigma^{-1}\|^2\|\bcE\|\left(
\|\bSigma - \hat{\bSigma}\|
\,+\,
\|\bSigma\|\frac{\|\bcE - \hat{\bcE}\|}{\|\bcE\|}\right)
~~(\text{from  $\balpha = \bSigma^{-1}\bcE$})\\
&\leq&
\|\bSigma^{-1}\|^2\left(
\|\bcE\|\|\bSigma - \hat{\bSigma}\|
\,+\,
\|\bSigma\|{\|\bcE - \hat{\bcE}\|}\right)\\
&\leq&
\|\bSigma^{-1}\|^2\left(\sqrt{Lm}
\|\bSigma - \hat{\bSigma}\|
\,+\,
\|\bSigma\|{\|\bcE - \hat{\bcE}\|}\right)
~~(\text{as  $\perf^\Dtrue[h] \in [0,1]$})\\
&\leq&
\mathcal{O}\left(\frac{1}{\sigma_{\min}(\bSigma)^2}\left(
\sqrt{Lm}\sqrt{\frac{L^2m\log(Lm/\delta)}{n^\tr}}
\,+\,
\sigma_{\max}(\bSigma)\sqrt{\frac{Lm\log(Lm/\delta)}{n^\val}}\right)\right),
\end{eqnarray*}
where the last step follows from 
 an adaptation of Lemma \ref{lem:Sigma-diff} (where $\H$ contains the $Lm$ fixed classifiers in \eqref{eq:trivial-classifiers}) and from Lemma \ref{lem:Perf-diff}. The last statement holds with probability at least $1-\delta$ over draws of $S^\tr$ and $S^\val$. Substituting this bound back in \eqref{eq:lhs-fixed} completes the proof.
\end{proof}

We next provide a bound on the singular values of $\bSigma$ for a specialized setting where the the probing classifiers $h^{\ell,i}$ are set to \eqref{eq:trivial-classifiers},  the basis functions $\phi^\ell$'s divide the data into disjoint clusters, and the base classifier $\bar{h}$ is close to having ``uniform accuracies'' across all the clusters and classes.
\begin{lem}
Let $h^{\ell,i}$'s be defined as in \eqref{eq:trivial-classifiers}. Suppose for any $x$, $\phi^\ell(x) \in \{0,1\}$ and $\phi^\ell(x)\phi^{\ell'}(x) = 0, \forall \ell \ne \ell'$. Let $p_{\ell,i} = \E_{(x,y)\sim \mu}[\phi^\ell(x)\1(y=i)]$. 
Let $\bar{h}$ be such that
$\kappa - \tau \leq \Phi^{\Dshift,\ell}_i[\bar{h}] \leq \kappa, \forall \ell, i$ and for some $\kappa < \frac{1}{m}$ and $\tau < \kappa$. Then:
\vspace{-2pt}
\[
\sigma_{\max}(\bSigma) \,\leq\, L\max_{\ell,i}\,p_{\ell,i} \,+\, \Delta;~~~~~
\sigma_{\min}(\bSigma) \,\geq\, \epsilon(1-m\kappa)\min_{\ell,i}\,p_{\ell,i} \,-\, \Delta,
\]
\vskip -0.2cm
where $\displaystyle \Delta = Lm\tau\max_{\ell,i}p_{\ell,i}$.
\end{lem}
\vskip -0.2cm
\begin{proof}
We first write the matrix $\bSigma$ as $\bSigma = \bar{\bSigma} + \bE$, where
\[
\bar{\bSigma} = 
\begin{bmatrix}
p_{1,1}\left(\epsilon + (1-\epsilon)\kappa\right) &
p_{1,2}(1-\epsilon)\kappa & \ldots & p_{1,m}(1-\epsilon)\kappa &
p_{2,1}\kappa & \ldots & p_{L,m}\kappa\\
p_{1,1}(1-\epsilon)\kappa &
p_{1,2}\left(\epsilon + (1-\epsilon)\kappa\right) & \ldots & p_{1,m}(1-\epsilon)\kappa &
p_{2,1}\kappa & \ldots & p_{L,m}\kappa\\
&&&\vdots\\
p_{1,1}\kappa & p_{1,2}\kappa & \ldots & p_{1,m}\kappa & p_{2,1}\kappa & \ldots &p_{L,m}\left(\epsilon + (1-\epsilon)\kappa\right)
\end{bmatrix},
\]
and 
$\bE \in \R^{Lm\times Lm}$ with each
$\displaystyle|E_{(\ell,i), (\ell',i')}| \leq \max_{\ell,i}p_{\ell,i}\left(\kappa - \Phi^{\Dshift,\ell}_i[\bar{h}]\right) \leq \tau\max_{\ell,i}p_{\ell,i}$.

The matrix $\bar{\bSigma}$ can in turn be written as a product of a \textit{symmetric} matrix $\A\in \R^{Lm\times Lm}$ and a \textit{diagonal} matrix $\D\in \R^{Lm\times Lm}$:
\[
\bar{\bSigma} = \A\D,
\]
where
\[
\A = 
\begin{bmatrix}
\epsilon + (1-\epsilon)\kappa &
(1-\epsilon)\kappa & \ldots &  (1-\epsilon)\kappa &
\kappa & \ldots & \kappa\\
(1-\epsilon)\kappa &
\epsilon + (1-\epsilon)\kappa & \ldots & (1-\epsilon)\kappa &
\kappa & \ldots & \kappa\\
&&&\vdots\\
(1-\epsilon)\kappa &
(1-\epsilon)\kappa & \ldots & \epsilon + (1-\epsilon)\kappa &
\kappa & \ldots & \kappa\\
&&&\vdots\\
\kappa & \kappa & \ldots & \kappa & \epsilon + (1-\epsilon)\kappa & \ldots &(1-\epsilon)\kappa\\
&&&\vdots\\
\kappa & \kappa & \ldots & \kappa & (1-\epsilon)\kappa & \ldots &\epsilon + (1-\epsilon)\kappa
\end{bmatrix}, \text{and}
\]
\[
\D = \diag(p_{1,1},\ldots,p_{L,m}).
\]
We can then bound the largest and smallest singular values of $\bSigma$ in terms of those of $\A$ and $\D$. Using Weyl's inequality (see e.g.,\ \cite{stewart1998perturbation}), we have
\begin{equation*}
\sigma_{\max}(\bSigma) \leq \sigma_{\max}(\bar{\bSigma})  + \|\bE\| \leq
\|\A\|\|\D\| \,+\, \|\bE\|
= \sigma_{\max}(\A)\sigma_{\max}(\D)  + \|\bE\|.
\label{eq:sigma-max-}
\end{equation*}
and
\begin{equation*}
\sigma_{\min}(\bSigma) \geq
\sigma_{\min}(\bar{\bSigma}) \,-\, \|\bE\|
=
\frac{1}{\|\bar{\bSigma}^{-1}\|} \,-\, \|\bE\|\geq \frac{1}{\|\A^{-1}\|\|\D^{-1}\|} \,-\, \|\bE\|= \sigma_{\min}(\A)\sigma_{\min}(\D)\,-\, \|\bE\|.
\label{eq:sigma-min-}
\end{equation*}
Further, we have $\|\bE\| \leq \|\bE\|_F \leq  \displaystyle Lm\tau\max_{\ell,i}p_{\ell,i} = \Delta$, giving us:
\begin{equation}
\sigma_{\max}(\bSigma) \leq \sigma_{\max}(\A)\sigma_{\max}(\D)  + \Delta.
\vspace{-10pt}
\label{eq:sigma-max}
\end{equation}
\begin{equation}
\sigma_{\min}(\bSigma) \geq  \sigma_{\min}(\A)\sigma_{\min}(\D) - \Delta.
\label{eq:sigma-min}
\end{equation}
All that remains is to bound the singular values of $\bSigma$ and $\D$. Since $\D$ is a diagonal matrix, it's singular values are given by its diagonal entries:
\[
\sigma_{\max}(\D) = \max_{\ell,i}\,p_{\ell,i};~~~~\sigma_{\min}(\D) = \min_{\ell,i}\,p_{\ell,i}.
\]
The matrix $\A$ is symmetric and has a certain block structure. It's singular values are the same as the positive magnitudes of its Eigen values. We first write out it's $Lm$ Eigen vectors:

\[
\begin{matrix}
\x^{1,1} &= [&
\overbrace{1, -1, 0, \ldots, 0}^{m~\text{entries}},& \overbrace{ 0, \ldots, 0}^{m~\text{entries}}, &\ldots &\overbrace{ 0, \ldots, 0}^{m~\text{entries}}&]\\[5pt]
\x^{1,2} &= [&
{1, 0, -1, \ldots, 0}, &{ 0, \ldots, 0}, &\ldots & 0, \ldots, 0&]\\
&&&\vdots\\[5pt]
\x^{1,m-1} &= [&
{1, 0, 0, \ldots, -1},& { 0, \ldots, 0},& \ldots &{ 0, \ldots, 0}&]\\[5pt]
\x^{1,m} &= [&
{1, \ldots, 1},& { -1, \ldots, -1},& \ldots &{ 0, \ldots, 0}&]\\[5pt]
\x^{2,1} &= [&
{0, \ldots, 0},& { 1,-1,0, \ldots, 0},& \ldots& { 0, \ldots, 0}&]\\
&&&\vdots\\[5pt]
\x^{2,m-1} &= [&
 { 0, \ldots, 0},& {1, 0, 0, \ldots, -1},&\ldots &{ 0, \ldots, 0}&]\\[5pt]
\x^{2,m} &= [&
{-1, \ldots, -1},& { 1, \ldots, 1},& \ldots &{ 0, \ldots, 0}&]\\[5pt]
&&&\vdots\\[5pt]
\x^{L,1} &= [&
{0, \ldots, 0},& { 0, \ldots, 0},& \ldots& { 1,-1,0, \ldots, 0}&]\\[5pt]
&&&\vdots\\[5pt]
\x^{L,m-1} &= [&
{0, \ldots, 0},& { 0, \ldots, 0},& \ldots& { 1,0,0, \ldots, -1}&]\\[5pt]
\x^{L,m} &= [&
{1, \ldots, 1},& { 1, \ldots, 1},& \ldots& { 1, \ldots, 1}&]\\[5pt]
\end{matrix}
\]
One can then verify that the $Lm$  Eigen values of $\A$ are $\epsilon$ with a multiplicity of $(L-1)m$, $\epsilon(1-m\kappa)$ with a multiplicity of $m-1$ and $(L-\epsilon)m\kappa + \epsilon$ with a multiplicity of 1. Therefore: 
\[
\sigma_{\max}(\A) \leq L;~~~~\sigma_{\min}(\A) = \epsilon(1-m\kappa).
\]
Substituting the singular (Eigen) values of $\A$ and $\D$ into \eqref{eq:sigma-max} and \eqref{eq:sigma-min} completes the proof. 
\end{proof}
In the above lemma, the base classifier $\bar{h}$ is assumed to have roughly uniformly low accuracies for all classes and clusters, and the closer it is to having uniform accuracies, i.e.\ the smaller the value of $\tau$, the tighter are the bounds.

We have shown a bound on the singular values of $\bSigma$ for a specific setting where the basis functions $\phi^\ell$'s divide the data into disjoint clusters. When this is not the case (e.g.\ with overlapping clusters \eqref{eq:hardlcuster}, or soft clusters \eqref{eq:softlcuster}), the singular values of $\bSigma$ would depend on how correlated the basis functions are.

\section{Error Bound for FW-EG with Unknown $\psi$}
\label{app:complex-unknown}
In this section, we provide an error bound for Algorithm \ref{algo:FW-modified} for evaluation metrics of the form $\perf^\Dtrue[h] = \psi(C^\Dtrue_{11}[h],\ldots, C^\Dtrue_{mm}[h]),$ for a smooth, but \textit{unknown} $\psi: \R^m \> \R_+$. In this case, we do not have a closed-form expression for the gradient of $\psi$, but instead apply the example weight elicitation routine in Algorithm \ref{algo:weight-coeff} using probing classifiers chosen from within a small neighborhood around the current iterate $h^t$, where  $\psi$ is effectively linear. Specifically, we invoke Algorithm \ref{algo:weight-coeff} with the current iterate $h^t$ as the base classifier and with the radius parameter $\epsilon$ set to a small value. In the error bound that we state below for this version of the algorithm, we explicitly take into account the ``slack'' in using a local approximation to $\psi$ as a proxy for its gradient.
\begin{thm}[\textbf{Error Bound for FW-EG with unknown $\psi$}]
\label{thm:fw-eg-error-bound-unknown-psi}
Let $\perf^\Dtrue[h] = \psi(C^\Dtrue_{11}[h],\ldots, C^\Dtrue_{mm}[h])$ for an \emph{unknown} concave function $\psi: [0,1]^m \>\R_+$, which is $Q$-Lipschitz, and also $\lambda$-smooth w.r.t.\ the $\ell_1$-norm.  Let $\hat{\perf}^\val[h]=\psi(\hat{C}^\val_{11}[h],\ldots, \hat{C}^\val_{mm}[h])$. Fix $\delta \in (0, 1)$. 
Suppose Assumption \ref{asp:alpha-star} holds with slack $\nu$. Suppose for any linear metric $\sum_i\beta_i C^\Dtrue_{ii}[h]$, whose associated weight coefficients in the assumption is $\bar{\balpha}$  with $\|\bar{\balpha}\| \leq B$, the following holds.
For any $\delta \in (0,1)$, with probability $\geq 1-\delta$ over draw of $S^\tr$ and $S^\val$, 
 when the weight elicitation routine in Algorithm \ref{algo:weight-coeff} is given an 
 input metric $\hat{\perf}^\val$ with $|\hat{\perf}^\val - \sum_i\beta_i \hat{C}^\val_{ii}[h]| \leq \chi, \forall h$, it outputs coefficients $\hat{\balpha}$ such that
 $\|\hat{\balpha} -\bar{\balpha}\| \leq 
 \kappa(\delta, n^\tr, n^\val, \chi)
 $, for some function $\kappa(\cdot) > 0$. Let $B' = B + \sqrt{Lm}\,\kappa(\delta, n^\tr, n^\val, 2\lambda\epsilon^2)$. 
Assume $m \leq n^\val$.
 Then w.p.\  $\geq 1 - \delta$ over draws of $S^\tr$ and $S^\val$ from $\Dtrue$ and $\Dshift$ respectively, the classifier $\hat{h}$ output by Algorithm \ref{algo:FW-modified} with radius parameter $\epsilon$ after $T$ iterations satisfies:
\begin{eqnarray*}
\lefteqn{
\max_{h}\perf^D[h] - \perf^D[\hat{h}]\,\leq\, 
2QB'\E_x\left[\|\eta^\tr(x)-
\hat{\eta}^\tr(x)\|_1\right] +
 4Q\sqrt{Lm}\,\kappa(\delta/T, n^\tr, n^\val, 2\lambda\epsilon^2)
 }
 \\
 &&
 \hspace{5.5cm}
\,+\, 4Q\nu \,+\, 
\mathcal{O}\left(\lambda m\sqrt{\frac{m\log(n^\val)\log(m) + \log(m/\delta)}{n^\val}} + \frac{\lambda}{T}\right).
\end{eqnarray*}
\end{thm}

One can plug-in $\kappa(\cdot)$ with e.g. the error bound we derived for Algorithm \ref{algo:weight-coeff} in Theorem \ref{thm:alpha-diagonal-linear-conopt}, suitably modified to accommodate input metrics $\hat{\perf}^\val$ that may differ from the desired linear metric by at most $\chi$. Such modifications can be easily made to Theorem \ref{thm:alpha-diagonal-linear-conopt} and would result in an  additional term $\sqrt{Lm}\chi$  in the error bound to take into account the additional approximation errors in computing the right-hand side of the linear system in \eqref{eq:perf-emp}.

Before proceeding to prove Theorem \ref{thm:fw-eg-error-bound-unknown-psi}, we state a few useful lemmas. The following  lemma shows that because $\psi(\C)$ is $\lambda$-smooth, it is effectively linear within a small neighborhood around $\C$.
\begin{lem}
\label{lem:eps-linear-approx}
Suppose $\psi$ is $\lambda$-smooth w.r.t.\ the $\ell_1$-norm. 
For each iteration $t$ of Algorithm \ref{algo:FW-modified}, let ${\bupsilon}^t = \nabla\psi(\c^t)$ denote the true gradient of $\psi$ at $\c^t$.
Then for any classifier $h^\epsilon(x) = (1-\epsilon)h^t(x) + \epsilon h(x),$
\begin{eqnarray*}
    \left|\hat{\perf}^\val[h^\epsilon] - \hat{\perf}^\val[h^t] - \sum_{i}\upsilon^t_i \hat{C}^\val_{ii}[h^\epsilon]\right|
    &\leq& 2\lambda\epsilon^2.
\end{eqnarray*}
\end{lem}
\begin{proof}
For any randomized classifier $h^\epsilon(x) = (1-\epsilon)h^t(x) + \epsilon h(x),$
\allowdisplaybreaks
\begin{eqnarray*}
    \left|\hat{\perf}^\val[h^\epsilon] - \hat{\perf}^\val[h^t] - \sum_{i}\upsilon^t_i \hat{C}^\val_{ii}[h^\epsilon]\right| &=&
    \left|\psi(\diag(\hat{\C}^\val[h^\epsilon])) - \psi(\diag(\hat{\C}^\val[h^t])) -  \sum_{i}\upsilon^t_i \hat{C}^\val_{ii}[h^\epsilon]\right|\\
    &\leq& \frac{\lambda}{2}\|\diag(\hat{\C}^\val[h^\epsilon]) \,-\, \diag(\hat{\C}^\val[h^t])\|^2_1\\
    &=& \frac{\lambda}{2}\|\epsilon(\diag(\hat{\C}^\val[h]) \,-\, \diag(\hat{\C}^\val[h^t]))\|_1^2\\
    &=&  \frac{\lambda}{2}\epsilon^2\|\diag(\hat{\C}^\val[h]) \,-\, \diag(\hat{\C}^\val[h^t])\|_1^2\\
     &\leq&  \frac{\lambda}{2}\epsilon^2\left(\|\diag(\hat{\C}^\val[h])\|_1 \,+\, \|\diag(\hat{\C}^\val[h^t])\|_1\right)^2\\
    &\leq& \frac{\lambda}{2}\epsilon^2(2)^2 ~=~2\lambda\epsilon^2.
\end{eqnarray*}
Here the second line follows from the fact that $\psi$ is $\lambda$-smooth w.r.t.\ the $\ell_1$-norm, and $\bupsilon^t = \nabla\psi(\diag(\hat{\C}^\val[h^t]))$. The third line follows from linearity of expectations. The last line follows from the fact that the sum of the entries of a confusion matrix  (and hence the sum of its diagonal entries) cannot exceed 1.
\end{proof}

We next restate the error bounds for the call to \textbf{PI-EW} in line 9 and the corresponding bound on the approximation error in the linear maximizer $\hat{f}$ obtained.
\begin{lem}[Error bound for call to \textbf{PI-EW}  in line 9 with unknown $\psi$]
\label{lem:plugin-linear-unknown-psi}
For each iteration $t$ of Algorithm \ref{algo:FW}, let ${\bupsilon}^t = \nabla\psi(\c^t)$ denote the true gradient of $\psi$ at $\c^t$,  when the algorithm is run with an unknown $\psi$ that is $Q$-Lipschitz and $\lambda$-smooth w.r.t.\ the $\ell_1$-norm.
Let
 $\bar{\balpha}$ be the associated weighting coefficient for the linear metric $\sum_i {\upsilon}^t_i C^D_{ii}[h]$ (whose coefficients are unknown) in Assumption \ref{asp:alpha-star}, with $\|\bar{\balpha}\|_1 \leq B$, and with slack $\nu$. Fix $\delta>0$.
Suppose w.p. $\geq 1-\delta$ over draw of $S^\tr$ and $S^\val$, when the weight elicitation routine used in \emph{\textbf{PI-EW}} is called with the input metric $\hat{\perf}^\val[h] - \hat{\perf}^\val[h^t]$  with $|\hat{\perf}^\val[h] - \sum_i\upsilon_i \hat{C}^\val_{ii}[h]| \leq \chi, \forall h$, it
outputs coefficients $\hat{\balpha}$ such that
 $\|\hat{\balpha} -\bar{\balpha}\| \leq 
 \kappa(\delta, n^\tr, n^\val, \chi)
 $, for some function $\kappa(\cdot) > 0$.  Let $B' = B + \sqrt{Lm}\,\kappa(\delta, n^\tr, n^\val, 2\lambda\epsilon^2)$. 
 Then with the same probability, the classifier $\hat{h}$ output by \emph{\textbf{PI-EW}}  when called 
 by Algorithm \ref{algo:FW-modified}
 with metric $\hat{\perf}^\lin[h]= \hat{\perf}^\val[h] - \hat{\perf}^\val[h^t]$
 and radius $\epsilon$ satisfies:
\begin{eqnarray*}
{
\max_{h}\sum_i {\upsilon}^t_i C^D_{ii}[h] - \sum_i {\upsilon}^t_i C^D_{ii}[\hat{h}]} &\leq
Q\left(B'\E_x\left[\|\eta^\tr(x)-
\hat{\eta}^\tr(x)\|_1\right]
\,+\,
 2\sqrt{Lm}\,\kappa(\delta, n^\tr, n^\val, 2\lambda\epsilon^2)
 \,+\, 2\nu
 \right),
\end{eqnarray*}
where $\eta_i^\tr(x) = \P^\mu(y=i|x)$.
\end{lem}
\begin{proof}
The proof is the same as that of Lemma \ref{lem:plugin-linear} for the ``known $\psi$'' case, except that the  $\kappa(\cdot)$ guarantee for the call to weight elicitation routine in line 2 is different, and takes into account the fact that the input metric $\hat{\perf}^\val[h] - \hat{\perf}^\val[h^t]$ to the weight elicitation routine  is only a local approximation to the (unknown) linear metric $\sum_i\upsilon_i \hat{C}^\val_{ii}[h]$. We use Lemma \ref{lem:eps-linear-approx} to compute the value of slack $\chi$ in $\kappa(\cdot)$.
\end{proof} 

\begin{lem}[Approximation error in linear maximizer $\hat{f}$ in line 9 with unknown $\psi$]
\label{lem:lmo-unknown} 
For each iteration $t$ in Algorithm \ref{algo:FW-modified}, let $\bar{\c}^t = \diag(\C^D[h^t])$ and let $\bar{\bbeta}^t = \nabla\psi(\bar{\c}^t)$ denote the unknown gradient of $\psi$ evaluated at $\bar{\c}^t$.
Suppose the assumptions in Theorem \ref{thm:fw-eg-error-bound-unknown-psi} hold. 
Let $B' = B + \sqrt{Lm}\,\kappa(\delta, n^\tr, n^\val, 2\lambda\epsilon^2)$. 
Assume $m \leq n^\val$. Then w.p.\ $\geq 1 - \delta$ over draw of $S^\tr$ and $S^\val$ from $\Dshift$ and $\Dtrue$ resp., for any $t = 1, \ldots, T$, the classifier $\hat{f}$ returned by \emph{\textbf{PI-EW}} in line 9 satisfies:
\begin{eqnarray*}
\lefteqn{
\max_{h}\sum_{i}\bar{\beta}^t_{ii}C^\Dtrue_i[h] \,-\, 
\sum_{i}\bar{\beta}^t_iC^\Dtrue_{ii}[\hat{f}] \,\leq\,
QB'\E_x\left[\|\eta^\tr(x)-
\hat{\eta}^\tr(x)\|_1\right] 
\,+\, 2Q\nu 
}\\
 &&
 \hspace{1cm}
 +\,
 2Q\sqrt{Lm}\,\kappa\left(\textstyle\frac{\delta}{T}, n^\tr, n^\val, 2\lambda \epsilon^2\right)
 \,+\, 
 \mathcal{O}\left(\lambda m\sqrt{\frac{m\log(m)\log(n^\val) + \log(m/\delta)}{n^\val}}\right).
\end{eqnarray*}
\end{lem}
\begin{proof}
The proof is the same as that of Lemma \ref{lem:lmo} for the ``known $\psi$'' case, with the only difference being that we use Lemma \ref{lem:plugin-linear-unknown-psi} (instead of Lemma \ref{lem:plugin-linear}) to bound the linear maximization errors in equation \eqref{eq:approx-lin-max-last}.
\end{proof}

\begin{proof}[Proof of Theorem \ref{thm:fw-eg-error-bound-unknown-psi}]

The proof follows from plugging Lemma \ref{lem:lmo-unknown} into the Frank-Wolfe convergence guarantee in Lemma \ref{lem:fw} stated in Appendix \ref{app:proof-FW}.
\end{proof}

\section{Running Time of Algorithm \ref{algo:FW}}
\label{app:run-time}
We discuss how one iteration of FW-EG (Algorithm~\ref{algo:FW}) compares with  one iteration (epoch) of training a class-conditional probability estimate $\hat\eta^{\tr}(x) \approx \P^\Dshift(y=1|x)$. In each iteration of FW-EG, we create $Lm$ probing classifiers, where each probing classifier via~\eqref{eq:trivial-classifiers} only  requires perturbing the predictions of the base classifier $\bar h = h^t$ and hence requires $n^\tr + n^\val$ computations. After constructing the $Lm$ probing classifiers, FW-EG solves a system of linear equations with $Lm$ unknowns, where a na\"ive matrix inversion approach requires $O((Lm)^3)$ time. Notice that this can be further  improved with efficient methods, e.g., using state-of-the-art linear regression solvers. Then FW-EG creates a plugin classifier and combines the predictions with the Frank-Wolfe style updates, requiring $Lm(n^\tr + n^\val)$ computations. So, the overall time complexity for each iteration of  
FW-EG is $O\left(Lm(n^\tr + n^\val) + (Lm)^3\right)$. On the other hand, one iteration (epoch) of training $\hat\eta^{\tr}(x)$ requires $O(n^\tr Hm)$ time, where $H$ represents the total number of parameters in the underlying model architecture up to the penultimate layer.
For deep networks such as ResNets (Sections~\ref{ssec:cifar10} and~\ref{ssec:adience}), clearly, the run-time is dominated by the training of $\hat\eta^{\tr}(x)$, as long as $L$ and $m$ are relatively small compared to the number of parameters in the neural network. Thus our approach is reasonably faster than having to train the model for $\hat\eta^\tr$ in each iteration~\cite{jiang2020optimizing}, training the model (such as ResNets) twice~\cite{patrini2017making}, or making multiple forward/backward passes on the training and validation set requiring three times the time for each epoch compared to training $\hat\eta^\tr$~\cite{ren2018learning}. 

\section{Plug-in with Coordinate-wise Search Baseline}
\label{ssec:multiclass-plugin}

We describe the Plug-in [train-val] baseline used in Section \ref{sec:experiments}, which constructs a classifier $\widehat{h}(x) \,\in\, \argmax_{i \in [m]} w_i\hat{\eta}^\val_i(x)$, by tuning the weights $w_i \in \R$ to maximize the given metric on the validation set .
Note that there are $m$ parameters to be tuned, and a na\"{i}ve approach would be to use an $m$-dimensional grid search. Instead, we use a trick from \citet{hiranandani2019multiclass} to decompose this search into  an independent coordinate-wise search for each $w_i$. Specifically, one can estimate the relative weighting  $w_i/w_j$ between any pair of classes $i,j$ by constructing a classifier of the form
\[
h^{\zeta}(x) = \begin{cases}
i & \text{if}~~~\zeta \hat{\eta}^\tr_i(x) > (1-\zeta) \hat{\eta}^\tr_j(x)\\
j & \text{otherwise}
\end{cases},
\]
that  predicts either class $i$ or  $j$
based on which of these receives a higher (weighted) probability estimates, 
and (through a line search) finding the  parameter  $\zeta \in (0,1)$ for which $h^{\zeta}$ yields the highest validation metric:
\[
w_i/w_j \approx \argmax_{\zeta \in [0,1]} \hat{\perf}^\val[h^{\zeta}].
\]
By fixing $i$ to class $m$, and repeating this for classes $j \in [m-1]$, one can estimate  $w_j/w_m$ for each $j \in [m-1]$, and normalize the estimated related weights  to get estimates for $w_1, \ldots, w_m$. 

\section{Solving Constrained Satisfaction Problem in \eqref{eq:con-opt}}
\label{app:con-opt}

We describe some common special cases where one can easily identify classifiers $h^{\ell,i}$'s which satisfy the constraints in \eqref{eq:con-opt}.
We will make use of a pre-trained class probability model $\hat{\eta}_i^\tr(x) \approx \P^\Dshift(y=i|x)$, also used in Section \ref{sec:algorithms} to construct the plug-in classifier in Algorithm \ref{algo:linear-metrics}. The hypothesis class $\H$ we consider is the set of all plug-in classifiers obtained by post-shifting $\hat{\eta}^\tr$.

\begin{figure}[h]
\centering
\includegraphics[scale=0.75]{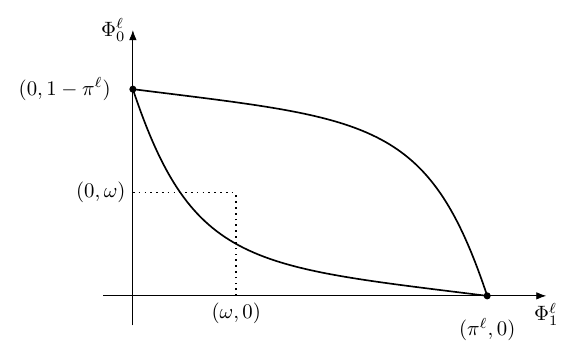}
\captionsetup{labelformat=empty}
\caption{Figure: Geometry of the space of $\Phi$-confusions~\cite{hiranandani2018eliciting} for $m=2$ classes and with basis functions $\phi^\ell(x) = \1(g(x) =\ell)$ which divide the data into $L$ disjoint clusters.
For a fixed cluster $\ell$, we plot the values of ${\Phi}^{\Dshift,\ell}_{0}[h]$ and ${\Phi}^{\Dshift,\ell}_{1}[h]$ for all randomized classifiers, with  $\pi^\ell = \P^{\Dshift}(y=1,g(x) = \ell)$. The points on the lower boundary correspond to  classifiers of the form $\1({\eta}^\tr(x) \leq 
\tau)$ for varying thresholds $\tau \in [0,1]$. 
The points on the lower boundary
within the dotted box correspond to the thresholded classifiers $h$ which yield both values $\Phi^{\Dshift,\ell}_0[h] \leq \omega$ and $\Phi^{\Dshift,\ell}_1[h] \leq \omega$. One can thus find a feasible probing classifier $h^{\ell,i}$ for the constrained optimization problem in~\eqref{eq:con-opt} using the construction given in \eqref{eq:binaryconopt}
as long as $\pi^{\ell} \geq \gamma$ and $1 - \pi^{\ell} \geq \gamma$, and the lower boundary intersects with the dotted box for clusters $\ell'\neq\ell$. If the latter does not happen, one can increase $\omega$ in small steps until the classifier given in~\eqref{eq:binaryconopt} is feasible for  \eqref{eq:con-opt}.
} 
\end{figure}

We start with a binary classification problem ($m = 2$) with basis functions $\phi^\ell(x) = \1(g(x) =\ell)$, which divide the data points into $L$ \textit{disjoint} groups according to $g(x) \in [L]$. For this setting, one can show under mild assumptions on the data distribution that \eqref{eq:con-opt} does indeed have a feasible solution (using e.g. the geometric techniques used by \citet{hiranandani2018eliciting} and also elaborated in the figure above).
One such feasible $h^{\ell, i}$  predicts class $i \in \{0,1\}$ on all example belonging to group $\ell$, and uses a thresholded of $\hat{\eta}^\tr$ for examples from other groups, with per-cluster thresholds. 
This would have the effect of maximizing the diagonal entry $\hat{\Phi}^{\tr,\ell}_i[h^{\ell, i}]$ of $\hat{\bSigma}$ and the thresholds can be tuned so that the off-diagonal entries $\hat{\Phi}^{\tr,\ell'}_{i'}[h^{\ell, i}], \forall (\ell',i') \ne (\ell,i)$ are small. More specifically, for any $\ell \in [L], i \in \{0,1\}$, the classifier $h^{\ell,i}$ can be constructed as:
\begin{equation}
h^{\ell, i}(x) = \begin{cases}
i &\text{ if  $g(x) = \ell$}\\
\1(\hat{\eta}^\tr(x) \leq 
\tau_{g(x)}) &\text{ otherwise},
\end{cases}
\label{eq:binaryconopt}
\end{equation}
where the thresholds $\tau_{\ell'} \in [0,1], \ell' \ne \ell$ can each be tuned independently using a line search to minimize $\max_{i'}\hat{\Phi}^{\tr,\ell'}_{i'}[h^{\ell, i}]$. As long as $\hat{\eta}^\tr$ is a close approximation of $\P(y|x)$, the above procedure is guaranteed to find an approximately feasible solution for \eqref{eq:con-opt}, provided one exists. Indeed one can tune the values of $\gamma$ and $\omega$ in \eqref{eq:con-opt}, so that the above construction (with tuned thresholds) satisfies the constraints.

We next look a multiclass problem ($m > 2$) with basis functions $\phi^\ell(x) = \1(g(x) =\ell)$ which again divide the data points into $L$ \textit{disjoint} groups. 
Here again, one can show under mild assumptions on the data distribution that \eqref{eq:con-opt} does indeed have a feasible solution (using e.g. the geometric tools from \citet{hiranandani2019multiclass}). 
We can once again construct a feasible $h^{\ell, i}$ 
by predicting class $i \in [m]$ on all example belonging to group $\ell$, and using a post-shifted classifier for examples from other groups. In particular, for any $\ell \in [L], i \in [m]$, the classifier $h^{\ell,i}$ can be constructed as:
\begin{equation}
h^{\ell, i}(x) =  
\begin{cases}
i &\text{ if  $g(x) = \ell$}\\
\argmax_{j\in [m]} w^{g(x)}_j\hat{\eta}_j^\tr(x) &\text{ otherwise}
\end{cases},
\label{eq:h-ell-i-con-opt}
\end{equation}
where we use $m$ parameters $w^{\ell'}_1, \ldots, w^{\ell'}_m$ for each cluster $\ell' \ne \ell$. 
We can then tune these $m$ parameters  to minimize the maximum of the off-diagonal entries of $\hat{\bSigma}$, i.e.\ {minimize} $\max_{i'}\hat{\Phi}^{\tr,\ell'}_{i'}[h^{\ell, i}]$. However, this may require an $m$-dimensional grid search. Fortunately, as described in Appendix \ref{ssec:multiclass-plugin}, we can use a trick from \citet{hiranandani2019multiclass} 
to reduce the problem of tuning $m$ parameters into $m$ independent line searches. This is based on the idea that the optimal relative weighting $w^{\ell'}_i/w^{\ell'}_j$ between any pair of classes can be determined through a line search. In our case, we will fix $w^{\ell'}_m = 1, \forall \ell' \ne \ell$ and compute  $w^{\ell'}_i, i = 1, \ldots, m-1$ by solving the following one-dimensional optimization problem to determine the relative weighting $w^{\ell'}_i / w^{\ell'}_m = w^{\ell'}_i$.
\[
w^{\ell'}_i \in \underset{\zeta \in [0,1]}{\argmin}\left( \max_{i'}\hat{\Phi}^{\tr,\ell'}_{i'}[h^{\zeta}]\right),
~~~~
\text{where}
~~~~
h^{\zeta}(x) = \begin{cases}
i & \text{if}~~~\zeta \hat{\eta}^\tr_i(x) < (1-\zeta) \hat{\eta}^\tr_m(x)\\
m & \text{otherwise}
\end{cases}.
\]
We can repeat this for each cluster $\ell' \ne \ell$ to construct the $(\ell,i)$-th probing classifier $h^{\ell,i}$ in \eqref{eq:h-ell-i-con-opt}.

For the more general setting, where the basis functions $\phi^\ell$'s cluster the data into overlapping or soft clusters (such as in \eqref{eq:softlcuster}),  one can find feasible classifiers for \eqref{eq:con-opt} by posing this problem as a ``rate'' constrained optimization problem of the form below to pick $h^{\ell,i}$:
\[
\max_{h\in \H}\hat{\Phi}^{\tr,\ell}_i[h]
~~\text{s.t.}~~
\hat{\Phi}^{\tr,\ell'}_{i'}[h] \leq \omega, \forall (\ell',i') \ne (\ell,i),
\]
which can be solved using off-the-shelf toolboxes such as the open-source library offered by \citet{cotter2019optimization}.\footnote{\url{https://github.com/google-research/tensorflow_constrained_optimization}}
Indeed one can tune the hyper-parameters $\gamma$ and $\omega$ so that the solution to the above problem is feasible for \eqref{eq:con-opt}.
If $\H$ is the set of plug-in classifiers obtained by post-shifting $\hat{\eta}^\tr$, then one can alternatively use the approach of \citet{narasimhan2018learning} to identify the optimal post-shift on $\hat{\eta}^\tr$ that solves the above constrained problem.

\section{Additional Experimental Details}
\label{app:exp}

\begin{table*}[t]
    \centering
    \caption{Data Statistics for different problem setups in Section~\ref{sec:experiments}.}
    {\footnotesize
    \begin{tabular}{lllll}
    \hline
        \textbf{Problem Setup} &\textbf{Dataset} & \textbf{\#Classes} &\textbf{\#Features} & \textbf{train / val / test split} \\
        \hline
        Indepen. Label Noise (Section~\ref{ssec:cifar10}) & CIFAR-10 & 10 & 32 $\times$ 32 $\times$ 3 & 49K / 1K / 10K  \\ \hline
        Proxy-Label (Section~\ref{ssec:adult}) & Adult & 2 & 101 & 32K / 350 / 16K  \\ \hline
        Domain-Shift (Section~\ref{ssec:adience}) & Adience & 2 & 256 $\times$ 256 $\times$ 3 & 12K / 800 / 3K \\ \hline
        Black-Box Fairness Metric (Section~\ref{ssec:adultbb}) & Adult & 2 (2 prot. groups) & 106 & 32K / 1.5K / 14K  \\
        \hline
    \end{tabular}}
    \label{tab:stats}
\end{table*}

For the experiments (Section~\ref{sec:experiments}), we provide the data statistics in Table~\ref{tab:stats}. Observe that we always use small validation data in comparison to the size of the training data. Below we provide some more details regarding the experiments:  

\begin{itemize}
    \item \textit{Maximizing Accuracy under Label Noise on CIFAR-10 (Section~\ref{ssec:cifar10}):} The metric that we aim to optimize  is test accuracy, which is a linear metric in the diagonal entries of the confusion matrix. Notice that we work with the \emph{asymmetric} label noise model from~\citet{patrini2017making}, which corresponds to the setting where a label is flipped to a particular label with a certain probability. This involves a non-diagonal noise transition matrix $\T$, and consequently the corrected  training objective  is a linear function of the entire confusion matrix. Indeed, the loss correction approach from~\cite{patrini2017making} makes use of the estimate of the entire noise-transition matrix, including the off-diagonal entries. Whereas, our approach in the experiment elicits weights for the diagonal entries alone, but assigns a different set of weights for each basis function, i.e., cluster. We are thus able to achieve better performance than \citet{patrini2017making}  by optimizing 
    correcting for the noise using a linear function of per-cluster diagonal entries.     Indeed, we also observed that PI-EW often achieves better accuracy during cross-validation with ten basis functions,  highlighting the benefit of underlying modeling in PI-EW. We expect to get further improvements by incorporating off-diagonal entries in PI-EW optimization on the training side as explained in Appendix~\ref{app:linear-gen}. We also stress that the results from our methods can  be further improved by cross-validating over kernel width, UMAP dimensions, and selection of the cluster centers, which are currently set to fixed values in our experiments. Lastly, we did not compare to the Adaptive Surrogates~\cite{jiang2020optimizing} for this experiment as this baseline requires to re-train the ResNet model in every iteration, and more importantly, this method constructs its probing classifiers by perturbing the parameters of the ResNet model several times in each iteration, which can be prohibitively expensive in practice. 
    
    \item \textit{Maximizing G-mean with Proxy Labels on Adult (Section~\ref{ssec:adult}):} In this experiment, we use binary features as basis functions instead of RBF kernels as done in CIFAR-10 experiment.  This reflects the flexibility of the proposed PI-EW and FW-EG methods. Our approach can incorporate any indicator features as basis function as long as it reflects cluster memberships. Moreover, our choice of basis function was motivated from choices made in~\cite{jiang2020optimizing}. We expect to further improve our results by incorporating more binary features as basis functions. 
    
    \item \textit{Maximizing F-measure under Domain Shift on Adience (Section~\ref{ssec:adience}):} As mentioned in Section~\ref{ssec:adience}, for the basis functions, in addition to the default basis $\phi^{\text{def}}(x) = 1 \, \forall x$, we choose from subsets of six basis  functions $\phi^1,\ldots,\phi^6$ that are averages of the RBFs, centered at points  from the validation set corresponding to each one of the six age-gender combinations. We choose these subsets using knowledge of the underlying image classification task. 
    Specifically, besides the default basis function, we cross-validate over three subsets of basis functions. The first subset comprises two basis functions, where the basis functions are averages of the RBF kernels with cluster centers belonging to the two true class. The second subset comprises three basis functions, where the basis functions are averages of the RBF kernels with cluster centers belonging to the three age-buckets. The third subset comprises six basis functions, where the basis functions are averages of the RBF kernels with cluster centers belonging to the combination of true class $\times$ age-bucket. We expect to further improve our results by cross-validating over kernel width and selection of the cluster centers. Lastly, we did not compare to Adaptive Surrogates, as this experiment again requires training a deep neural network model, and perturbing or retraining the model in each iteration can be prohibitively expensive in practice.
    
    \item \textit{Maximizing Black-box Fairness Metric on Adult (Section~\ref{ssec:adultbb}):}  In this experiment, since we treat the metric as a black-box, we do not assume access to gradients and thus do not run the [$\psi$ known] variant of FW-EG. We only report the [$\psi$ unknown] variant of FW-EG with varied basis functions as shown in Table~\ref{tab:adultbb}. 
    
    \item In Table~\ref{tab:addedexp}, we replicate the ``Macro F-measure'' experiment (without noise) from Section 6.2 in~\cite{jiang2020optimizing} and report results of maximizing the macro F-measure on Adult, COMPAS and Default datasets. We see that our approach yields notable gains on two out of the three datasets in comparison to Adaptive Surrogates approach~\cite{jiang2020optimizing}. 
    
    \begin{table}[h]
    \small
    \centering
    \caption{Test macro F-measure for the maximization task in Section 6.2 of~\citet{jiang2020optimizing}.}
    \vskip -0.2cm
    \begin{tabular}{ccc}
    \hline
     $\downarrow$ Data, Method $\rightarrow$&
         Adaptive Surrogates~\cite{jiang2020optimizing} & FW-EG \\
        \hline
        COMPAS 
        & 0.629 & \textbf{0.652}
        \\
        Adult 
        & 0.665 & \textbf{0.670} 
        \\
        Default
        & 0.533 & {0.536}
        \\
        \hline
    \end{tabular}
    \label{tab:addedexp}
    \vskip -0.3cm
\end{table}
    
\end{itemize}

\end{appendices}